\documentclass[journal,10pt]{IEEEtran}

\usepackage{amsmath,amssymb,amsfonts,amsthm}
\usepackage{graphicx}
\usepackage{booktabs}
\usepackage{array}
\usepackage{multirow}
\usepackage{algorithm}
\usepackage{algpseudocode}
\usepackage{tikz}
\usepackage{xcolor}
\usepackage{url}
\usepackage[hidelinks]{hyperref}

\newtheorem{theorem}{Theorem}
\newtheorem{lemma}[theorem]{Lemma}
\newtheorem{proposition}[theorem]{Proposition}
\newtheorem{corollary}[theorem]{Corollary}
\theoremstyle{definition}
\newtheorem{definition}[theorem]{Definition}
\newtheorem{assumption}[theorem]{Assumption}
\theoremstyle{remark}
\newtheorem{remark}[theorem]{Remark}

\newcommand{\Hc}{\mathcal{H}}
\newcommand{\Xc}{\mathcal{X}}
\newcommand{\Yc}{\mathcal{Y}}
\newcommand{\Dc}{\mathcal{D}}
\newcommand{\Cc}{\mathcal{C}}
\newcommand{\Risk}{R}
\newcommand{\Rstar}{R^{\star}}
\newcommand{\Hblind}[1]{\Hc^{\perp #1}}
\newcommand{\Hsig}[1]{\Hc^{#1}}
\newcommand{\price}{\mathcal{P}}
\newcommand{\head}{\Delta}
\newcommand{\VR}{\mathrm{VR}}
\newcommand{\Fone}{F_1}
\newcommand{\MAJ}{\textsf{MAJORITY}}
\newcommand{\GEO}{\textsf{GEO\_ONLY}}
\newcommand{\WX}{\textsf{WX\_ONLY}}
\newcommand{\MC}{\textsf{MC}}
\newcommand{\MONOnv}{\textsf{MONO\_naive}}
\newcommand{\MONOsc}{\textsf{MONO\_sc}}
\newcommand{\PHYSnv}{\textsf{PHYS}}
\newcommand{\PHYSsc}{\textsf{PHYS\_sc}}
\newcommand{\MIB}{\textsf{MIB}}
\newcommand{\MIBfp}{\textsf{MIB\_fp}}
\newcommand{\MIBfm}{\textsf{MIB\_fm}}
\DeclareMathOperator*{\argmin}{arg\,min}

\begin{document}

\title{The Cost of a Physics Prior Is Bounded \newline by the Ablation Gap}

\author{Boris~Kriuk%
\thanks{B.~Kriuk is with the Department of Computer Science and Engineering,
The Hong Kong University of Science and Technology (HKUST), Clear Water Bay,
Kowloon, Hong Kong SAR, China (e-mail: bkriuk@connect.ust.hk).}%
\thanks{Code, run configurations and raw artefacts:
\url{https://github.com/BorisKriuk/The-Cost-of-a-Physics-Prior-Is-Bounded-by-the-Ablation-Gap}.}}

\markboth{}{}

\IEEEtitleabstractindextext{%
\begin{abstract}
Physics-informed and shape-constrained machine learning routinely reports an
\emph{accuracy cost of enforcing a prior} and treats that number as a property of
the prior. We show it is primarily a property of the free features and the
validation split. Let $\price$ be the excess risk of restricting a hypothesis
class to functions obeying a shape constraint on a feature subset $S$, and let
$\head$ be the excess risk of the \emph{ablated} model that ignores $S$. Because a
function constant in $x_j$ is simultaneously non-decreasing and non-increasing in
$x_j$, the ablated class is contained in the constrained class, so
$0\le\price\le\head$ for every risk functional, with no convexity, smoothness or
realizability assumption. In the language of an experiment, the bound collapses to
a single sign test: \emph{a constrained model must never be beaten by its own
ablation}. We instantiate this on an ordinal wildfire-severity task
($N=26\,681$, $K=3$) with hard monotone constraints on four meteorological
drivers, geographic coordinates left free, and a three-rung validation ladder from
i.i.d.\ resampling to $2^\circ$ spatial blocking, across a nine-configuration
grid. Four consequences are non-obvious. First, coordinates act as a
\emph{shield}: alone they recover $92.9\%$ of the full model's macro-$\Fone$ under
spatial blocking, which collapses $\head$ from $0.1288$ to $0.0427$ and forces
every prior on the drivers to look cheap; the same screened prior costs $0.0473$
shielded and $0.3470$ unshielded, a ratio of $7.3$ with identical physics.
Second, because $\head$ is protocol-dependent it is not transferable: coarsening
blocks from $1^\circ$ to $10^\circ$ drives $\head$ from $0.0942$ to $0.0050$, so
two of our spatial configurations were unidentifiable before any constrained model
was trained. Third, empirical inversions of the certified nesting bound the
pipeline's additive resolution in the reported metric, and maximizing over $318$
certified comparisons yields a self-calibrating floor $\hat s=0.0220$
macro-$\Fone$, below which no price cell can be interpreted --- including four of
twenty-one in our own headline configuration. Fourth, constraint cost and
constraint compliance are logically independent: the unconstrained model violates
the textbook prior at rate $0.48$--$0.49$ while enforcing it costs $0.0473$. We
close with a two-fit screen that rejects unidentifiable constraint-cost
experiments before any constrained model is trained.
\end{abstract}

\begin{IEEEkeywords}
Shape constraints, monotone classification, ordinal regression, physics-informed
machine learning, inductive bias, spatial cross-validation, ablation analysis,
measurement resolution, wildfire severity.
\end{IEEEkeywords}}

\maketitle
\IEEEdisplaynontitleabstractindextext
\IEEEpeerreviewmaketitle

\section{Introduction}
\IEEEPARstart{A}{} recurring experiment in scientific machine learning proceeds as
follows. A practitioner holds a model and a piece of domain knowledge: fire
severity should not decrease with temperature, nor increase with relative
humidity; permafrost fraction should fall as air temperature rises; aftershock
intensity should decay with elapsed time. The knowledge is imposed as a hard shape
constraint or a soft penalty, the model is retrained, and the change in a held-out
metric is reported as \emph{the cost of the prior}. When small, the prior is
declared essentially free, and the conclusion is generalized: physics can be
respected at negligible expense~\cite{karniadakis2021,willard2022,raissi2019}.

Such claims appear throughout the physics-informed literature and its
application-facing branches, including our own earlier hybrid physics--ML risk
model for pan-Arctic permafrost infrastructure, which blends learned
climate--permafrost relationships with a prescribed physical sensitivity in a
fixed mixture to preserve extrapolative behaviour~\cite{kriuk2026permafrost};
and POSEIDON, which embeds Gutenberg--Richter and Omori--Utsu scaling as learnable
constraints and reports that the physical parameters land in scientifically
plausible ranges while enhancing rather than compromising predictive
accuracy~\cite{kriuk2026poseidon}. Parallel claims are standard in the
monotonicity-constrained tabular
literature~\cite{gupta2016,you2017,liu2020,cano2019}. The reported quantity is
typically of the form ``$\Delta \Fone = -0.01$'' or ``$R^2$ unchanged to three
decimals.''

This paper makes a structural observation that changes how such numbers must be
read, and then --- more importantly --- shows that the observation is already
computed by any study that runs an ablation baseline.

Consider a constraint requiring $f$ to be non-decreasing in $x_j$. A function that
does not depend on $x_j$ is non-decreasing in $x_j$ --- weakly, but the constraint
is stated weakly --- and is therefore \emph{admissible}. The same function is also
non-increasing in $x_j$, hence admissible under the opposite prior, and admissible
under convexity, concavity, Lipschitz bounds and unimodality in $x_j$. So the worst
a shape constraint on $S$ can do is force the learner to discard $S$
(Fig.~\ref{fig:nesting}):
\begin{equation}
\underbrace{\price(\sigma)}_{\text{price of the prior on }S}
\;\le\;
\underbrace{\head(S)}_{\text{value of }S\text{ given the rest}} .
\label{eq:teaser}
\end{equation}

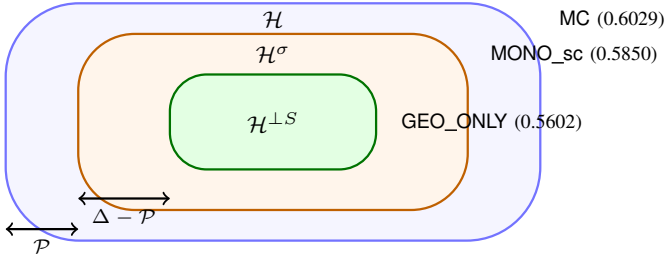
\begin{figure}[!t]
\centering
\resizebox{\columnwidth}{!}{%
\begin{tikzpicture}[scale=0.95]
  \draw[rounded corners=26pt, fill=blue!4, draw=blue!55, thick]
        (-3.5,-1.55) rectangle (3.5,1.55);
  \draw[rounded corners=20pt, fill=orange!8, draw=orange!75!black, thick]
        (-2.55,-1.15) rectangle (2.55,1.15);
  \draw[rounded corners=13pt, fill=green!11, draw=green!45!black, thick]
        (-1.35,-0.62) rectangle (1.35,0.62);
  \node at (0,0) {\footnotesize $\Hblind{S}$};
  \node at (0,0.88) {\footnotesize $\Hsig{\sigma}$};
  \node at (0,1.34) {\footnotesize $\Hc$};
  \node[anchor=west] at (1.55,0.0)  {\scriptsize \textsf{GEO\_ONLY} (0.5602)};
  \node[anchor=west] at (2.72,0.88) {\scriptsize \textsf{MONO\_sc} (0.5850)};
  \node[anchor=west] at (3.62,1.34) {\scriptsize \textsf{MC} (0.6029)};
  \draw[<->,thick] (-1.35,-1.0) -- (-2.55,-1.0)
       node[midway,below,font=\scriptsize] {$\head-\price$};
  \draw[<->,thick] (-2.55,-1.4) -- (-3.5,-1.4)
       node[midway,below,font=\scriptsize] {$\price$};
\end{tikzpicture}}
\caption{Geometry behind Theorem~\ref{thm:main}, with the model names of our
experimental zoo attached and their measured out-of-fold macro-$\Fone$ under
$2^\circ$ spatial blocking. The class $\Hblind{S}$ of hypotheses that ignore the
constrained features $S$ lies \emph{inside} the shape-constrained class
$\Hsig{\sigma}$, because a function constant in $x_j$ satisfies every sign pattern
on $x_j$ vacuously. Hence the constrained optimum can never be worse than the
ablated optimum, and the price of the constraint is bounded by the ablation gap:
here $\price=0.0179\le\head=0.0427$. The bound holds for arbitrary risk
functionals.}
\label{fig:nesting}
\end{figure}

We do not claim \eqref{eq:teaser} is deep; its proof is three lines. We claim that
its right-hand side is a quantity constrained-learning papers \emph{already
measure}, under the name ``ablation'' or ``feature-importance baseline,'' and that
reading the two numbers together dissolves several confusions. Concretely, in our
own experimental harness the bound reduces to the statement that the column
$\Delta\Fone$ versus the ablated baseline must be non-negative for every
constrained model --- a sign check on a column the pipeline already writes.

\subsection{Contributions}
\begin{enumerate}
\item \textbf{A certified ceiling} (Theorem~\ref{thm:main},
Corollary~\ref{cor:general}). For any risk functional and any class closed under
feature ablation, the excess risk of a degenerate-admissible constraint family on
$S$ is at most the ablation gap of $S$. This covers hard constraints, penalized
objectives (Corollary~\ref{cor:soft}), monotonicity under arbitrary partial
orders, convexity, Lipschitz and unimodality constraints. It is attained
(Proposition~\ref{prop:sharp}) and so cannot be improved.

\item \textbf{An operational reformulation}
(Corollary~\ref{cor:dominance}): \emph{ablation dominance}. The bound is
equivalent to requiring each constrained model to be no worse than its own
ablation. This yields a falsifier computable from an existing results table with
no additional training.

\item \textbf{Shielding} (Proposition~\ref{prop:shield}), measured. If any
$S$-blind surrogate is near-optimal, $\head(S)\approx0$ and the constraint is
automatically almost free, \emph{regardless of whether the prior is correct}. In
our design the shield is $\{\mathrm{lat},\mathrm{lon}\}$; it alone attains
$92.9\%$ of the full model's macro-$\Fone$ under spatial blocking, and it makes
the identical monotone prior appear $7.3\times$ cheaper than when it is removed.

\item \textbf{Bounded sign controversy} (Corollary~\ref{cor:spread}). The entire
spread of achievable risk over all $2^{|S|}$ sign patterns, and over the choice of
sign for an auxiliary channel, is bounded by the same $\head$. We report the one
regime in which our measured spread exceeds the bound, and show that the excess
decomposes exactly into two elementary nesting inversions of magnitude
$\le\hat s$.

\item \textbf{A self-calibrating noise floor} (Corollary~\ref{cor:slack}).
Since \eqref{eq:teaser} holds at the population level, an empirical inversion of
any certified nesting measures the pipeline's own resolution. Maximizing over
$318$ certified comparisons in our grid gives $\hat s=0.0220$ macro-$\Fone$.
Crucially $\hat\head<\hat s$ is checkable \emph{before} any constrained model is
trained (Algorithm~\ref{alg:screen}); it fires on two of our twenty-one
configuration-protocol cells.

\item \textbf{Cost and compliance are independent}
(Proposition~\ref{prop:vr}). We construct families with $\price=0$ and violation
rate approaching $1$, and we measure the same dissociation: the unconstrained
model violates the textbook prior at rate $0.4795$--$0.4825$ while the price of
enforcing that prior is $0.0143$--$0.0549$.

\item \textbf{Empirical calibration} on an ordinal wildfire-severity task with
hard monotone constraints, four drivers, a three-rung validation ladder, an
exogeneity gate on auxiliary channels, and paired cluster bootstrap inference
across a nine-configuration grid
(Sections~\ref{sec:instantiation}--\ref{sec:results}).
\end{enumerate}

\subsection{What this paper is not}
It is not a new constrained-learning algorithm, and it does not argue that shape
constraints are useless. They buy extrapolation, auditability and regulatory
defensibility~\cite{wang2020deontological,gupta2016}, none of which
in-distribution accuracy measures. It argues that the \emph{accuracy cost} of such
constraints, as currently reported, is largely a statement about the free features
and the split.

\section{Related Work}
\subsection{Monotone and shape-constrained learning}
Monotonicity as an inductive bias goes back to
Ben-David~\cite{bendavid1995} and Sill's monotonic networks~\cite{sill1997}, and
has developed into partially monotone neural networks~\cite{daniels2010},
calibrated interpolated lookup tables~\cite{gupta2016}, deep lattice
networks~\cite{you2017}, and certification for unconstrained
architectures~\cite{liu2020}. Monotone splits are standard in production gradient
boosting~\cite{friedman2001,chen2016,ke2017}, which is what makes hard constraints
cheap to impose in our study. Cano \emph{et al.}~\cite{cano2019} survey algorithms
and evaluation measures for monotonic classification. Across this literature the
accuracy penalty is reported per dataset; to our knowledge it is never reported
alongside the ablation gap of the constrained features, which
Theorem~\ref{thm:main} shows is the only quantity that makes it interpretable.

\subsection{Physics-informed and hybrid modeling}
Physics-informed neural networks~\cite{raissi2019} and the broader programme of
embedding scientific knowledge into learning~\cite{karniadakis2021,willard2022}
motivate constraints from conservation laws and empirical scaling relations. Our
own prior work makes the ``inexpensive physics'' claim explicit in two settings:
the pan-Arctic permafrost risk framework, where learned relationships are mixed
with a prescribed physical temperature sensitivity specifically to survive
extrapolation~\cite{kriuk2026permafrost}; and POSEIDON, where seismological
scaling laws enter as learnable constraints~\cite{kriuk2026poseidon}. The present
paper is complementary rather than adversarial to both: it identifies the
additional number ($\head$, computed over the covariates that the prescribed
component touches) that turns such a statement from unfalsifiable into
informative, and the regime ($\head<\hat s$) in which it cannot be tested at all.
The same accounting applies to self-configuring learners whose inductive biases
are set by the algorithm rather than the user, such as adaptive tree
morphing~\cite{kriuk2025morphboost} and epigenetically adapted evolutionary
search~\cite{kriuk2025elena}, where the free-feature set itself moves during
training; and to the broader question of which stage of adaptation an inductive
bias belongs to~\cite{kriuk2026aai}.

\subsection{Validation protocol and spatial structure}
That random $k$-fold cross-validation inflates apparent skill on spatially
autocorrelated data is well
established~\cite{roberts2017,ploton2020,valavi2019,meyer2021,meyer2019}. The
usual conclusion is that reported accuracies are optimistic --- and we reproduce
it: our unconstrained model falls from $0.8316$ to $0.6029$ macro-$\Fone$ between
i.i.d.\ and $2^\circ$-blocked splits. We add a second-order consequence: because
protocol choice changes $\head$, it changes the \emph{admissible range of every
constraint-cost figure}, so protocol sensitivity propagates into claims about
inductive bias and interpretability, not merely into headline metrics. Related
cautions about correlation- and attention-derived structure under non-i.i.d.\
resampling appear in~\cite{kriuk2025deepsupp,kriuk2026orca}. The wildfire corpus
and its meteorological joins follow~\cite{kriuk2025eurasia}.

\section{Setup and Notation}
\label{sec:setup}
Let $\Xc=\prod_{j=1}^{d}\Xc_j$ with each $\Xc_j\subseteq\mathbb{R}$ totally
ordered, let $\Yc=\{0,\dots,K-1\}$ carry the natural order, and let $\Dc$ be a
distribution on $\Xc\times\Yc$. Write $V=\{1,\dots,d\}$ for the available
features, $S\subseteq V$ for the constrained subset, and $x_{-S}$ for the
coordinates outside $S$.

\begin{definition}[Risk functional]
A \emph{risk} is any functional $\Risk:\Hc\to\mathbb{R}\cup\{+\infty\}$ with lower
values preferred. For $\mathcal{G}\subseteq\Hc$ put
$\Rstar(\mathcal{G})=\inf_{f\in\mathcal{G}}\Risk(f)$.
\end{definition}

No structure is assumed on $\Risk$. It may be an expected loss, a negated
macro-$\Fone$, an AUC deficit, a quantile of a per-region error distribution, or a
cross-validation estimate of any of these. This generality is what makes the
result usable for the metrics applied papers actually report; our own metric,
macro-$\Fone$, is not an expected loss and is not decomposable over examples.

\begin{definition}[Ordinal scoring function]
A hypothesis is a pair $(f,\text{dec})$ where $f:\Xc\to\mathbb{R}$ is a score and
$\text{dec}$ a fixed decoding of $f$ into $\Yc$. Concretely we use the cumulative
parameterization $f(x)=\sum_{k=0}^{K-2}\hat{S}_k(x)$ with
$\hat{S}_k(x)\approx\Pr[y>k\mid x]$ and $\hat S_0\ge\cdots\ge\hat S_{K-2}$
enforced by running isotonization in $k$; the label is the argmax of the induced
class probabilities. Shape constraints are imposed on each $\hat S_k$ and hence on
$f$, and isotonization in $k$ preserves monotonicity in $x$.
\end{definition}

\begin{definition}[Sign pattern and shape class]
For $\sigma\in\{-1,0,+1\}^{d}$ with support $S(\sigma)=\{j:\sigma_j\neq0\}$,
\[
\Hsig{\sigma}=\bigl\{f\in\Hc:\ \forall j\in S,\ \forall x_{-j},\
t\mapsto\sigma_j f(x_{-j},t)\ \text{non-decr.}\bigr\}.
\]
\end{definition}

\begin{definition}[$S$-blind class]
$\Hblind{S}=\{f\in\Hc:\ f(x)=f(x')\ \text{whenever}\ x_{-S}=x'_{-S}\}$.
\end{definition}

\begin{assumption}[Ablation closure]
\label{as:closure}
$\Hblind{S}\neq\emptyset$.
\end{assumption}

Assumption~\ref{as:closure} is satisfied by every class used in practice. It is
worth verifying it constructively for the class we use, since the verification is
what makes the theorem non-vacuous here: a histogram gradient-boosted tree
ensemble that never splits on any coordinate in $S$ is constant in $S$, and
therefore satisfies \emph{any} per-feature monotonicity specification on $S$,
including contradictory ones. The constrained solver is thus free to reach the
ablated model, and does so whenever the constraint is sufficiently costly.

\begin{definition}[Price and headroom]
\label{def:pq}
\begin{align}
\price(\sigma)&=\Rstar(\Hsig{\sigma})-\Rstar(\Hc),\\
\head(S)&=\Rstar(\Hblind{S})-\Rstar(\Hc).
\end{align}
\end{definition}

Both are non-negative by nesting~\cite{vapnik1998}. $\head(S)$ is exactly the
number an applied paper reports when it writes ``removing temperature costs $0.13$
macro-$\Fone$.''

\begin{definition}[Violation rate]
For $f\in\Hc$, coordinate $j$, step $h_j>0$, and a distribution $\mu$ on $\Xc$,
let $D_j(x)=f(x+h_je_j)-f(x-h_je_j)$ and let
$\tau=\varepsilon\,\mathrm{range}(f)$ be an activity threshold. Then
\[
\VR(f;j,\sigma_j)=
\Pr_{x\sim\mu}\!\left[\operatorname{sign} D_j(x)\neq\sigma_j \;\middle|\; |D_j(x)|>\tau\right].
\]
We distinguish $\VR_{\text{partial}}$, evaluated on the model's internal score, from
$\VR_{\text{e2e}}$, evaluated end-to-end through any preprocessing the deployed
pipeline applies; the latter can be positive when the former is zero if a
preprocessing stage reintroduces dependence on $S$. Throughout we use
$h_j=0.05\times\mathrm{IQR}(x_j)$ and $\varepsilon=10^{-3}$.
\end{definition}

\section{Main Result}

\begin{lemma}[Degenerate admissibility]
\label{lem:degen}
For every $\sigma\in\{-1,0,+1\}^{d}$ with support $S$,
$\Hblind{S}\subseteq\Hsig{\sigma}$.
\end{lemma}

\begin{proof}
Let $f\in\Hblind{S}$, $j\in S$, and fix $x_{-j}$. For $t\le t'$,
$f(x_{-j},t)=f(x_{-j},t')$ since $f$ does not depend on coordinate $j\in S$.
Hence $\sigma_j\bigl(f(x_{-j},t')-f(x_{-j},t)\bigr)=0\ge0$, so
$t\mapsto\sigma_jf(x_{-j},t)$ is weakly non-decreasing. As $j\in S$ was arbitrary,
$f\in\Hsig{\sigma}$.
\end{proof}

\begin{theorem}[The price of a shape constraint is bounded by its headroom]
\label{thm:main}
Under Assumption~\ref{as:closure}, for every risk functional $\Risk$ and every
$\sigma$ with support $S$,
\begin{equation}
0\;\le\;\price(\sigma)\;\le\;\head(S).
\label{eq:main}
\end{equation}
Equivalently $\Rstar(\Hsig{\sigma})\le\Rstar(\Hblind{S})$.
\end{theorem}

\begin{proof}
By Lemma~\ref{lem:degen}, $\Hblind{S}\subseteq\Hsig{\sigma}\subseteq\Hc$. The
infimum over a larger set is no larger, so
$\Rstar(\Hc)\le\Rstar(\Hsig{\sigma})\le\Rstar(\Hblind{S})$. Subtract
$\Rstar(\Hc)$.
\end{proof}

\begin{remark}[On triviality]
The proof uses nothing beyond monotonicity of the infimum. We state it as a
theorem not because it is hard but because its right-hand side is routinely
computed and never connected to its left-hand side. The content of this paper is
in Corollaries~\ref{cor:dominance}--\ref{cor:slack} and
Proposition~\ref{prop:vr}, and in the calibration of $\hat s$ in
Section~\ref{sec:slack}.
\end{remark}

\subsection{The operational form: ablation dominance}

Write the metric in higher-is-better form, $\Fone$, so that
$\price=\Fone(\text{unconstrained})-\Fone(\text{constrained})$ and
$\head=\Fone(\text{unconstrained})-\Fone(\text{ablated})$.

\begin{corollary}[Ablation dominance]
\label{cor:dominance}
Inequality \eqref{eq:main} is equivalent to
\begin{equation}
\Fone(\text{constrained})\;\ge\;\Fone(\text{ablated}),
\label{eq:dominance}
\end{equation}
i.e.\ the constrained model must not be beaten by the model that discards the
constrained features. In particular the test requires no knowledge of
$\Fone(\text{unconstrained})$ and no subtraction of correlated estimates.
\end{corollary}

\begin{proof}
Subtract the two displayed definitions: $\price\le\head$ iff
$-\Fone(\text{constrained})\le-\Fone(\text{ablated})$.
\end{proof}

Corollary~\ref{cor:dominance} is the form we recommend for practice, because it is
a single sign check on a difference that experiment harnesses already tabulate,
and because it is numerically better behaved: it compares two quantities directly
rather than differencing each against a common reference.

\subsection{How far the bound generalizes}

\begin{definition}
A family $\Cc\subseteq\Hc$ is \emph{$S$-degenerate-admissible} if
$\Hblind{S}\subseteq\Cc$.
\end{definition}

\begin{corollary}[General shape constraints]
\label{cor:general}
If $\Cc$ is $S$-degenerate-admissible then
$\Rstar(\Cc)-\Rstar(\Hc)\le\head(S)$. This covers, for each $j\in S$:
monotonicity of either sign and isotonicity with respect to any partial order on
$\Xc_S$; convexity or concavity; $L$-Lipschitz dependence for any $L\ge0$;
unimodality, quasi-convexity and bounded total variation; bounded sensitivity
$\|\partial f/\partial x_j\|_\infty\le\kappa$; and pairwise dominance or
individual-fairness constraints $x\preceq x'\Rightarrow f(x)\le f(x')$.
\end{corollary}

\begin{proof}
Each listed property holds for every function constant in the coordinates of $S$:
a constant is non-decreasing and non-increasing, convex and concave, $0$-Lipschitz,
unimodal, of zero total variation, has vanishing partial derivative, and satisfies
every dominance requirement with equality. Hence $\Hblind{S}\subseteq\Cc$ and
Theorem~\ref{thm:main} applies verbatim.
\end{proof}

\begin{corollary}[Penalized and soft constraints]
\label{cor:soft}
Let $\Omega:\Hc\to[0,\infty)$ be a penalty with $\Omega\equiv0$ on $\Hblind{S}$,
and $\Risk_\lambda=\Risk+\lambda\Omega$, $\lambda\ge0$. Then
$\Rstar_\lambda(\Hc)-\Rstar(\Hc)\le\head(S)$ for every $\lambda\ge0$.
\end{corollary}

\begin{proof}
$\Risk_\lambda=\Risk$ on $\Hblind{S}$, so
$\Rstar_\lambda(\Hc)\le\inf_{\Hblind{S}}\Risk_\lambda=\Rstar(\Hblind{S})$.
\end{proof}

This matters because most physics-informed pipelines use soft penalties rather
than hard feasibility~\cite{krishnapriyan2021}: the ceiling is independent of the
penalty weight, so no amount of $\lambda$-tuning can make the physics term cost
more than the constrained features are worth.

\begin{proposition}[Sharpness]
\label{prop:sharp}
The bound is attained. Let $d=1$, $\Xc_1=\{0,1\}$, $\Yc=\{0,1\}$, $\Dc$ uniform on
$\{(0,1),(1,0)\}$, $\Hc$ all real functions on $\Xc$,
$\Risk(f)=\Pr[\mathbf 1\{f(x)>0\}\neq y]$, $\sigma=(+1)$. Then $\Rstar(\Hc)=0$ and
$\price(\sigma)=\head(\{1\})=\tfrac12$.
\end{proposition}

\begin{proof}
$f(x)=\tfrac12-x$ attains risk $0$. Any non-decreasing $f$ induces a
non-decreasing decision rule, which must assign a label to $x=1$ no smaller than
to $x=0$; the Bayes labels are $y(0)=1,y(1)=0$, so every such rule errs on at
least one of two equiprobable atoms, giving $\Rstar(\Hsig\sigma)=\tfrac12$. Every
$S$-blind $f$ is constant and errs on exactly one atom, so
$\Rstar(\Hblind{S})=\tfrac12$.
\end{proof}

Proposition~\ref{prop:sharp} identifies the worst case: tightness occurs precisely
when the prior is maximally \emph{anti-aligned} with the mechanism. Hence a
measured $\hat\price$ close to $\hat\head$ is evidence \emph{against} the physics;
$\hat\price\ll\hat\head$ is weak evidence for it; and $\hat\head\approx0$ is
evidence about the feature set, not the physics. Section~\ref{sec:shield} exhibits
all three regimes in one corpus: $\hat\price/\hat\head=0.82$ unshielded,
$0.37$ shielded, and $\hat\head=0.0050$ under coarse blocking.

\section{Consequences}

\subsection{Constraint costs do not transfer across protocols}
Write $\Risk_\Pi$ for the risk induced by an evaluation protocol $\Pi$, and
$\price_\Pi,\head_\Pi$ accordingly.

\begin{corollary}[Protocol non-transferability]
\label{cor:protocol}
For every $\Pi$, $\price_\Pi(\sigma)\le\head_\Pi(S)$. Consequently: (i) if
$\head_{\Pi'}(S)\le\varepsilon$ then $\price_{\Pi'}(\sigma)\le\varepsilon$ for
\emph{every} prior supported on $S$, including priors contradicting the data;
(ii) $\price_\Pi$ carries no information about $\price_{\Pi'}$ beyond the joint
behaviour of $\head_\Pi,\head_{\Pi'}$, and the ratio
$\price_{\Pi'}/\price_\Pi$ is unbounded above and below.
\end{corollary}

\begin{proof}
Apply Theorem~\ref{thm:main} with $\Risk=\Risk_\Pi$ and $\Risk=\Risk_{\Pi'}$
separately; cross-validation estimators are risk functionals. (i) is immediate.
For (ii), Proposition~\ref{prop:sharp} realizes $\price=\head$ and
Proposition~\ref{prop:shield} realizes $\price=0$ for the same $\sigma$; embedding
one behaviour in $\Pi$ and the other in $\Pi'$ gives arbitrary ratios.
\end{proof}

Two papers reporting ``monotonicity costs $0.01$'' under random splits and
``monotonicity costs $0.09$'' under blocked splits are therefore not in
disagreement, and are not measuring the same thing. In our grid the same prior on
the same corpus costs $0.0456$ at $1^\circ$ blocks and $-0.0139$ at $10^\circ$
blocks, a sign reversal driven entirely by the split.

\subsection{Shielding: how to make a prior look free}

\begin{proposition}[Shielding]
\label{prop:shield}
If there exists $f^\dagger\in\Hblind{S}$ with
$\Risk(f^\dagger)\le\Rstar(\Hc)+\varepsilon$, then $\head(S)\le\varepsilon$ and
hence $\price(\sigma)\le\varepsilon$ for every $\sigma$ supported on $S$.
Moreover $\head(S)$ is non-increasing under enlargement of the free feature set
$V\setminus S$, and non-decreasing under enlargement of $S$.
\end{proposition}

\begin{proof}
The first claim is Definition~\ref{def:pq} with Theorem~\ref{thm:main}. For the
second, $\Rstar(\Hc)$ is fixed once $V$ is fixed, while $\Rstar(\Hblind{S})$ is
non-increasing as the free set grows and non-decreasing as $S$ grows.
\end{proof}

The uncomfortable corollary for practice: if the feature set contains a proxy that
reconstructs the signal carried by $S$ --- coordinates standing in for
climatology, a station identifier for local calibration, a timestamp for
seasonality~\cite{meyer2019,geirhos2020,kaufman2012} --- then $\head(S)$ is small
and the prior is guaranteed cheap, unconditionally on whether it is correct.
Reporting a small $\hat\price$ in such a configuration is not evidence that the
model respects physics; it is evidence that the model did not need the constrained
features. We call the free proxy set the \emph{shield}; in our design the shield is
$\{\mathrm{lat},\mathrm{lon}\}$ and its strength is measured directly by the
ablated baseline. The monotonicity in $|S|$ is also observed: promoting insolation
from auxiliary channel to fifth driver raises $\hat\head$ from $0.1288$ to
$0.1399$ under i.i.d.\ splitting.

\subsection{The spread over sign patterns is bounded too}

\begin{corollary}[Bounded sign controversy]
\label{cor:spread}
Let $\Sigma$ be any set of sign patterns supported on $S$. Then
\begin{equation}
\sup_{\sigma\in\Sigma}\Rstar(\Hsig\sigma)-\inf_{\sigma\in\Sigma}\Rstar(\Hsig\sigma)
\;\le\;\head(S).
\label{eq:spread}
\end{equation}
The same holds for a single auxiliary channel $c\notin S$: with
$\Cc_0\supseteq\Cc_{+}\cup\Cc_{-}$ the classes obtained by leaving the sign of $c$
free, fixing it to $+1$, and to $-1$, all three contain the $c$-blind class, so
\begin{equation}
\max_{\text{sign}}\Fone-\min_{\text{sign}}\Fone
\;\le\;\Fone(\Cc_0)-\Fone(\text{$c$-blind}).
\label{eq:ctxspread}
\end{equation}
\end{corollary}

\begin{proof}
All $\Hsig\sigma$ with $\sigma$ supported on $S$ satisfy
$\Hblind{S}\subseteq\Hsig\sigma\subseteq\Hc$, so their optima lie in the interval
$[\Rstar(\Hc),\Rstar(\Hblind{S})]$ of length $\head(S)$. For the second claim,
$\Cc_{\pm}$ and $\Cc_0$ all contain the $c$-blind class by
Lemma~\ref{lem:degen}, and $\Cc_\pm\subseteq\Cc_0$; hence the maximum over signs
is $\Fone(\Cc_0)$ and the minimum is at least $\Fone(\text{$c$-blind})$.
\end{proof}

Corollary~\ref{cor:spread} has a deflationary reading we consider useful. A
substantial literature is devoted to \emph{which} sign a physical prior should
take when the textbook sign and the data-driven sign disagree. Whatever the
resolution, the total accuracy at stake is capped by the ablation gap. If the
ablation gap is below the noise floor, the disagreement cannot be adjudicated by
accuracy at all --- only by extrapolation behaviour or external theory.
Section~\ref{sec:ctxarm} reports the empirical spread, the seven cells in which it
exceeds the bound, and the exact decomposition of each excess into two elementary
inversions.

\subsection{Slack as an instrument}

Theorem~\ref{thm:main} is an identity about infima. Empirically we fit one model
per class by regularized empirical risk minimization on each training fold and
evaluate a pooled out-of-fold statistic; the fitted model is not the minimizer of
that statistic over its class. Every certified nesting therefore yields a
\emph{testable ordering} whose empirical inversion measures the combined
finite-sample estimation and optimization gap.

\begin{definition}[Certified comparison]
\label{def:cert}
A pair $(A,B)$ of arms is a \emph{certified comparison} if the hypothesis class of
$A$ is contained in that of $B$, so that $\Fone(B)\ge\Fone(A)$ at the population
level. Its \emph{inversion} is $\bigl(\Fone(A)-\Fone(B)\bigr)_+$.
\end{definition}

\begin{corollary}[Slack lower-bounds pipeline error]
\label{cor:slack}
Define
\begin{equation}
\hat s=\max_{(A,B)\in\mathcal{K}}\bigl(\Fone(A)-\Fone(B)\bigr)_{+}
\label{eq:slack}
\end{equation}
over all certified comparisons in a grid $\mathcal{K}$ of configurations. Then
$\hat s$ lower-bounds the additive error of the pipeline in the reported metric,
and:
\begin{enumerate}
\item \textbf{(reporting floor)} any $|\hat\price_c|<\hat s$ is unresolvable and
must not be interpreted;
\item \textbf{(a priori futility)} if $\hat\head_c<\hat s$ then configuration $c$
cannot resolve the price of \emph{any} prior on $S$, since the whole admissible
interval $[0,\head_c]$ lies inside the noise;
\item \textbf{(pre-training screen)} the test in (2) needs only the unconstrained
fit and its $S$-ablation, neither of which requires a constrained solver.
\end{enumerate}
\end{corollary}

\begin{proof}
At the population level $\Fone(A)\le\Fone(B)$. Writing
$\hat\Fone(A)=\Fone(A)+\eta_A$ and $\hat\Fone(B)=\Fone(B)+\eta_B$, a positive
observed difference gives
$\hat\Fone(A)-\hat\Fone(B)\le\eta_A-\eta_B\le|\eta_A|+|\eta_B|$; maximize over
$\mathcal{K}$. Claims (1)--(3) follow by comparing $[0,\hat\head_c]$ with
$\hat s$. Restricting $\mathcal{K}$ to ablation pairs recovers
Corollary~\ref{cor:dominance} and gives the narrower floor
$\hat s_{\mathrm{abl}}\le\hat s$.
\end{proof}

Unlike a bootstrap confidence interval, $\hat s$ derives from an ordering the
experiment is \emph{obliged} to satisfy, so it captures finite-sample estimation
error and constrained-solver suboptimality simultaneously --- error sources that
resampling the data leaves invisible~\cite{nadeau2003}. It does not capture seed
variance~\cite{bouthillier2021}: all runs use seed $0$, and \textsf{b2} confirms
only bytewise determinism at that seed. A seed sweep would plausibly raise
$\hat s$ and disqualify further cells, so we regard $\hat s=0.0220$ as a lower
bound on the pipeline's resolution, and we do not correct it for multiplicity
across the grid~\cite{demsar2006}. We regard this as the most practically useful
consequence of Theorem~\ref{thm:main}. A corollary of the corollary: a grid
$\mathcal K$ containing no heavily shielded configuration will report $\hat s=0$
vacuously, so probing configurations should be included deliberately. Our grid
(Table~\ref{tab:grid}) includes a block-size sweep spanning $1^\circ$ to
$10^\circ$ and a gate-forced-open configuration that adds a free channel; the
maximizing comparison turns out to lie in the headline configuration itself.

\begin{algorithm}[!t]
\caption{Admissibility screen for a constraint-cost experiment}
\label{alg:screen}
\begin{algorithmic}[1]
\Require features $V$, constrained subset $S$, protocol $\Pi$, metric $\Fone$,
slack $\hat s$ (or a grid $\mathcal K$ to estimate it)
\State fit unconstrained model on $V$; record $\Fone^{\mathrm{unc}}$
\State fit ablated model on $V\setminus S$; record $\Fone^{\mathrm{abl}}$
\State $\hat\head\gets \Fone^{\mathrm{unc}}-\Fone^{\mathrm{abl}}$
       \Comment{certified ceiling, Thm.~\ref{thm:main}}
\If{$\hat\head<\hat s$}
  \State \textbf{report} ``price of any prior on $S$ is unidentifiable under $\Pi$''
  \State \textbf{stop} \Comment{no constrained training required}
\EndIf
\State fit constrained model; record $\Fone^{\mathrm{con}}$
\If{$\Fone^{\mathrm{con}}<\Fone^{\mathrm{abl}}-\hat s$}
  \State \textbf{flag} solver or estimation failure
       \Comment{violates Cor.~\ref{cor:dominance}}
\EndIf
\State \textbf{report} $(\hat\price,\hat\head,\hat s)$ \emph{and} the free
       feature set $V\setminus S$
\end{algorithmic}
\end{algorithm}

\subsection{Compliance is not price}

\begin{proposition}[Independence of violation rate and price]
\label{prop:vr}
For every $\delta>0$ there exist $(\Dc,\Hc,\sigma,\mu)$ with $\price(\sigma)=0$
and $\VR(\hat f;j,\sigma_j)\ge1-\delta$ for some risk minimizer
$\hat f\in\argmin_{\Hc}\Risk$. Symmetrically there exist instances with
$\VR(\hat f;j,\sigma_j)=0$ and $\price(\sigma')>0$ for a different prior on the
same coordinate.
\end{proposition}

\begin{proof}[Proof sketch]
Take $d=2$, $y$ a deterministic function of $x_2$ alone, $x_1$ independent of
$(x_2,y)$. Every $f(x)=g(x_2)+\epsilon h(x_1)$ attains the same risk for a
decision-based metric provided $\epsilon$ is small enough not to change any
argmax. The minimizer is therefore not unique, and one may choose $h$ strictly
decreasing on a set of $\mu$-measure $\ge1-\delta$, giving $\VR\ge1-\delta$ under
$\sigma_1=+1$. Since $\Hblind{\{1\}}$ contains a risk-optimal hypothesis,
$\head(\{1\})=0$ and Theorem~\ref{thm:main} forces $\price(\sigma)=0$. The
symmetric statement follows by making $y$ depend monotonically on $x_1$ and
choosing the opposite sign.
\end{proof}

Proposition~\ref{prop:vr} dissolves a common rhetorical move. Some papers argue a
model ``has learned the physics'' because enforcing the physics is cheap; others
argue the opposite because the unconstrained model violates the physics often.
Both inferences are invalid: $\VR$ does not appear in \eqref{eq:main}, and the two
quantities can be set independently. A high $\VR$ with a small $\price$ is the
generic situation whenever the constrained features are partially redundant, and it
is exactly what we measure (Section~\ref{sec:compliance}).

\section{Instantiation: Mapping the Theory onto a Model Zoo}
\label{sec:instantiation}
The bound is only useful if the abstract classes correspond to models one actually
fits. Table~\ref{tab:mapping} gives the correspondence used throughout our
experiments. All arms share one architecture --- a cumulative ordinal ensemble of
$K-1$ histogram gradient-boosted binary classifiers~\cite{friedman2001} with
identical hyperparameters ($300$ iterations, learning rate $0.06$, $31$ leaves,
minimum $40$ samples per leaf, $L_2$ regularization $1.0$, no early stopping,
seed $0$) --- and differ only in feature set and in the per-feature monotonicity
specification, which is a hard constraint of the split finder rather than a
penalty. Consequently the nesting of Fig.~\ref{fig:nesting} holds \emph{exactly}
at the level of representable functions, and no penalty weight is tuned anywhere.

\begin{table}[!t]
\renewcommand{\arraystretch}{1.2}
\caption{Correspondence between the theory of Section~\ref{sec:setup} and the
fitted models. $D=\{\textit{temp},\textit{rh},\textit{wind},\textit{precip}\}$
are the meteorological drivers, $G=\{\mathrm{lat},\mathrm{lon}\}$ the geographic
shield, $c$ the residualized auxiliary channel (insolation).
$\sigma^{\mathrm{nv}}=(+1,-1,+1,-1)$ is the textbook (naive) sign pattern,
$\sigma^{\mathrm{sc}}$ the fold-wise screened one.}
\label{tab:mapping}
\centering
\footnotesize
\begin{tabular}{@{}llll@{}}
\toprule
Model & Features & Constraint & Role \\
\midrule
\MAJ        & ---        & ---                     & $\Hblind{D}$ for \textsf{PHYS*} \\
\GEO        & $G$        & none                    & $\Hblind{D}$ for \textsf{MONO*}, \textsf{MIB*} \\
\WX         & $D$        & none                    & $\Hc$ for \textsf{PHYS*} \\
\MC         & $D\cup G$  & none                    & $\Hc$ for \textsf{MONO*} \\
\MONOnv     & $D\cup G$  & $\sigma^{\mathrm{nv}}$ on $D$ & $\Hsig{\sigma^{\mathrm{nv}}}$ \\
\MONOsc     & $D\cup G$  & $\sigma^{\mathrm{sc}}$ on $D$ & $\Hsig{\sigma^{\mathrm{sc}}}$; $c$-blind arm \\
\PHYSnv     & $D$        & $\sigma^{\mathrm{nv}}$ on $D$ & unshielded prior \\
\PHYSsc     & $D$        & $\sigma^{\mathrm{sc}}$ on $D$ & unshielded prior \\
\MIB        & $D\cup G\cup\{c\}$ & $\sigma^{\mathrm{sc}}$ on $D$ & $\Cc_0$ in \eqref{eq:ctxspread} \\
\MIBfp      & $D\cup G\cup\{c\}$ & $\sigma^{\mathrm{sc}}$, $c{:}+1$ & $\Cc_{+}$ \\
\MIBfm      & $D\cup G\cup\{c\}$ & $\sigma^{\mathrm{sc}}$, $c{:}-1$ & $\Cc_{-}$ \\
\bottomrule
\end{tabular}
\end{table}

Four families of certified ceilings follow, and they are the falsifiable content
of the empirical section.

\begin{proposition}[Instantiated ceilings]
\label{prop:ceilings}
At the population level, for every protocol:
\begin{align}
\Fone(\textsf{MONO*}),\ \Fone(\textsf{MIB*})
   &\ge \Fone(\GEO), \tag{R1}\label{eq:R1}\\
\Fone(\textsf{PHYS*})
   &\ge \Fone(\textsf{best constant}), \tag{R2}\label{eq:R2}\\
\Fone(\textsf{MIB*})
   &\ge \Fone(\MONOsc), \tag{R3}\label{eq:R3}\\
\Fone(\MIB)
   &\ge \Fone(\MIBfp),\ \Fone(\MIBfm), \tag{R4}\label{eq:R4}\\
\Fone(\MC)
   &\ge \Fone(\textsf{MONO*}), \notag\\
\Fone(\WX)
   &\ge \Fone(\textsf{PHYS*}). \tag{R5}\label{eq:R5}
\end{align}
Moreover
$\max_{\text{ctx sign}}\Fone-\min_{\text{ctx sign}}\Fone
\le\Fone(\MIB)-\Fone(\MONOsc)$.
\end{proposition}

\begin{proof}
\eqref{eq:R1}--\eqref{eq:R3} are Corollary~\ref{cor:dominance} applied to the
corresponding rows of Table~\ref{tab:mapping}: within the $D\cup G$ architecture
$\Hblind{D}$ is exactly the class of ensembles that never split on $D$, i.e.\ the
\GEO{} class, and within the $D$-only architecture it is the constants; the
$c$-blind subclass of any \textsf{MIB} arm is the \MONOsc{} class, since an
ensemble that never splits on $c$ satisfies any sign specification on $c$.
\eqref{eq:R4} is $\Cc_\pm\subseteq\Cc_0$ and \eqref{eq:R5} is
$\Hsig\sigma\subseteq\Hc$. The final claim is \eqref{eq:ctxspread}.
\end{proof}

\begin{remark}[\eqref{eq:R2} is deliberately conservative]
\label{rem:cons}
We test a \emph{loosened} ceiling, so that a violation is correspondingly stronger
evidence of pipeline error. Under macro-$\Fone$ the majority-class predictor is not
optimal among constants --- predicting a rare class can raise the macro average ---
so $\Fone(\MAJ)\le\Fone(\textsf{best constant})$. \eqref{eq:R1} is conservative for
the \textsf{MIB} arms for the same reason: their true $D$-ablation is
``coordinates plus a free channel,'' which contains the \GEO{} class.
\end{remark}

\begin{remark}[\MIB{} versus \MC{} is not a constraint cost]
\label{rem:notmatched}
\MIB{} carries a feature \MC{} does not, so neither class contains the other and
no certified ordering holds between them. Differences we report in that column are
the \emph{joint} effect of the $D$-prior and the added channel, not a price; an
unconstrained model on $D\cup G\cup\{c\}$ was not fitted, so a feature-matched
reference for the \textsf{MIB} arms does not exist in this zoo. This is why
\MIB{} exceeding \MC{} in Table~\ref{tab:main} is not an inversion, and why we
exclude those cells from \eqref{eq:slack}.
\end{remark}

\begin{remark}[A degeneracy check built into the harness]
\label{rem:degencheck}
When the exogeneity gate of Section~\ref{sec:design} rejects all auxiliary
candidates, the channel $c$ is absent and Lemma~\ref{lem:degen} predicts that
\MIB{}, \MIBfp{} and \MIBfm{} must be \emph{identical} to \MONOsc{}, not merely
close: constraining the sign of a feature that does not exist is vacuous. Our
implementation asserts bytewise equality of out-of-fold predictions in this case
and aborts otherwise. The assertion passes in all three protocols of both
gate-empty configurations (\textsf{closed} and \textsf{drv}), and signature
deduplication reports the arms as aliases rather than as small insignificant
differences. This converts a corollary of the theorem into a runtime test of the
constraint plumbing, and it is the cheapest such test we know.
\end{remark}

\section{Experimental Design}
\label{sec:design}
The experiments are not intended to demonstrate a new method; they calibrate the
quantities of Theorem~\ref{thm:main} and measure $\hat s$ on a real environmental
prediction task where monotone physical priors are uncontroversial.

\subsection{Data and task}
We use the Eurasian wildfire corpus of~\cite{kriuk2025eurasia}: georeferenced fire
records joined to reanalysis meteorology at the time and place of detection. The
target is an \emph{ordered} severity/type variable with $K=3$ levels, declared in
increasing severity as \textsf{Controlled burn} $\prec$
\textsf{Uncontrolled burn} $\prec$ \textsf{Wildfire}, so the ordinal cumulative
parameterization of Section~\ref{sec:setup} applies. After class filtering and
completeness filtering on drivers and coordinates, $N=26\,681$ records remain,
with class counts $[333,\,2711,\,23\,637]$ (shares $0.0125$, $0.1016$, $0.8859$)
and a majority-class macro-$\Fone$ baseline of $0.3132$. Drivers are
$D=\{\textit{temp},\textit{rh},\textit{wind},\textit{precip}\}$; the shield is
$G=\{\mathrm{lat},\mathrm{lon}\}$.

Two properties of the corpus constrain what this design can test: the absence of a
parsable date column, so there is no temporal rung on the ladder and calendar
harmonics are unavailable to the exogeneity gate; and the near-surjectivity of
coordinates on records, so grouping by exact coordinate removes little information
relative to i.i.d.\ resampling. Effectively the ladder has two informative rungs:
unblocked (\textsf{iid}, \textsf{point}) and spatially blocked.

\subsection{Priors}
The textbook pattern is $\sigma^{\mathrm{nv}}=(+1,-1,+1,-1)$ on
$(\textit{temp},\textit{rh},\textit{wind},\textit{precip})$. The screened pattern
$\sigma^{\mathrm{sc}}$ is obtained \emph{inside each training fold} as the sign of
the mean central finite-difference partial effect of the unconstrained \MC{}
model, with a $400$-replicate clustered bootstrap interval reported per fold so
that non-identified signs are visible rather than silently adopted. Screening on
the training fold only is essential: screening on the pooled data would leak the
test fold into the constraint specification. In the headline configuration
$\sigma^{\mathrm{sc}}$ coincided bytewise with $\sigma^{\mathrm{nv}}$ in $9$ of
$15$ training folds ($2/5$ under \textsf{iid}, $4/5$ under \textsf{point}, $3/5$
under \textsf{spatial}), and \MONOsc{} was consequently fitted as \MONOnv{} in
those folds, detected by signature deduplication.

\subsection{Auxiliary channel and its exogeneity gate}
Candidate auxiliary channels (dewpoint, apparent temperature, pressure, cloud
cover, insolation, and calendar harmonics when a parsable date exists) are
admitted only if they are not reconstructible from the drivers. Each candidate $c$
is regressed on a quadratic design in $D$ --- intercept, linear, squares and all
pairwise products --- and admitted iff $R^2\le0.90$ and the residual relative
standard deviation is $\ge0.10$. Only one candidate is present in this corpus,
insolation, and it is admitted comfortably: $R^2=0.2646$, residual relative
standard deviation $0.8576$. Admitted channels enter the \textsf{MIB} arms in
\emph{residualized} form, $c-A\hat b$, with $\hat b$ estimated on training data
only. The motivation is that a partially monotone certificate on $D$ is
meaningless if a free channel is a smooth function of $D$: the deployed pipeline
can then reintroduce arbitrary dependence on $D$ through the free channel, which
is exactly what $\VR_{\text{e2e}}$ detects. The gate opened in every one of the
$15$ training folds of every gate-auto configuration, a fact that turns out to
neutralize our screening-leak probe (Section~\ref{sec:leak}).

\subsection{Validation ladder}
Three protocols in increasing severity: \textsf{iid} stratified $5$-fold;
\textsf{point}, grouping by exact coordinate so no location is split across folds;
and \textsf{spatial}, grouping by blocks of a chosen angular size ($429$ blocks at
$2^\circ$, $1080$ at $1^\circ$, $115$ at $5^\circ$, $40$ at $10^\circ$). Grouped
folds use stratified group $k$-fold. This ladder is the standard remedy for
autocorrelation-inflated skill~\cite{roberts2017,ploton2020,valavi2019}.

\subsection{Configuration grid}
Because $\hat s$ in \eqref{eq:slack} is a maximum over configurations, we run the
full ladder under each of the settings in Table~\ref{tab:grid}. Nine tags were
executed; three of them (\textsf{ctx}, \textsf{b2}, \textsf{glob}) produce
bytewise identical point estimates, leaving \emph{seven numerically distinct}
configurations. We designate \textsf{ctx} the headline configuration; all figures
are generated from it.

\begin{table}[!t]
\renewcommand{\arraystretch}{1.2}
\caption{Configuration grid $\mathcal K$. Each row is a complete run of the
three-protocol ladder over all eleven arms of Table~\ref{tab:mapping}, with
$5$ folds, seed $0$ and $B=2000$ cluster bootstrap replicates. \textsf{ctx},
\textsf{b2} and \textsf{glob} coincide bytewise; see Section~\ref{sec:leak}.}
\label{tab:grid}
\centering
\footnotesize
\begin{tabular}{@{}lll@{}}
\toprule
Tag & Setting & Purpose \\
\midrule
\textsf{ctx}    & headline: $2^\circ$ blocks, gate auto/fold & headline numbers \\
\textsf{b1}     & $1^\circ$ blocks ($1080$)   & shield strongest \\
\textsf{b2}     & $2^\circ$ blocks ($429$)    & determinism check ($\equiv$\textsf{ctx}) \\
\textsf{b5}     & $5^\circ$ blocks ($115$)    & block sweep \\
\textsf{b10}    & $10^\circ$ blocks ($40$)    & headroom collapse; $\hat s$ probe \\
\textsf{closed} & gate forced closed          & \textsf{MIB*}$\equiv$\MONOsc{} test \\
\textsf{drv}    & insolation as $5$th driver  & driver-set perturbation, $|S|{=}5$ \\
\textsf{glob}   & gate decided on pooled data & screening-leak probe \\
\textsf{k5}     & $K{=}4$ label granularity   & label-resolution perturbation \\
\bottomrule
\end{tabular}
\end{table}

The \textsf{k5} configuration re-splits the majority class, yielding
$K=4$ ordered levels \textsf{Controlled burn} $\prec$
\textsf{Uncontrolled burn} $\prec$ \textsf{Forest fire} $\prec$
\textsf{Natural fire} with counts $[333,\,2711,\,18\,415,\,5195]$ on
$N=26\,654$ records and a majority baseline of $0.2043$. It is included because
sign identifiability turns out to depend far more on label resolution than on
protocol.

\subsection{Inference}
All paired comparisons use one shared set of $B=2000$ cluster bootstrap
replicates~\cite{efron1993}, resampling spatial blocks, so that every pairwise
difference is computed on identical replicates and is therefore directly
comparable; replicates missing a class are discarded. All $2000$ replicates
survived in every configuration and protocol, so no interval in this paper rests
on a degenerate resample. Predictions are pooled out-of-fold and coverage is
asserted. Models whose predictions coincide bytewise are detected by signature
deduplication, fitted once, and reported as identical rather than as a small
insignificant difference --- a distinction that matters because
Lemma~\ref{lem:degen} predicts exact coincidence in specific configurations. Note
that the bootstrap clusters are the same spatial blocks that define the
\textsf{spatial} folds, so the block-size sweep changes interval widths on the
\textsf{iid} and \textsf{point} rungs even though their point estimates are
invariant; the half-width of the \MONOsc{} price under \textsf{iid} grows from
$0.0228$ at $2^\circ$ to $0.0246$ at $10^\circ$.

\section{Results}
\label{sec:results}
Table~\ref{tab:main} is the headline configuration in full; Table~\ref{tab:grid_f1}
is the whole grid. Six runtime falsifiers passed in every configuration:
\WX{} present in all cost tables (C1); intervals computed against all three
tabular references (C2); $\VR_{\text{own}}=0$ for every constrained arm (C3a) and
undefined for every unconstrained arm (C3b); no bootstrap degeneration (C4);
$\textsf{MIB*}\equiv\MONOsc{}$ bytewise whenever the gate is empty (C5); and
complete out-of-fold coverage (C6).

\subsection{Protocol ladder and the collapse of headroom}
Fig.~\ref{fig:ladder} shows macro-$\Fone$ for the key arms across the ladder. The
unconstrained model falls from $0.8316$ (\textsf{iid}) through $0.8311$
(\textsf{point}) to $0.6029$ (\textsf{spatial}), a split-leakage term of
$+0.2287$; under $10^\circ$ blocks it reaches $0.5218$, a term of $+0.3098$.

The quantity that governs everything downstream is not that drop but the vertical
gap between \MC{} and \GEO{}, i.e.\ $\hat\head(D)$: it falls from $0.1288$ under
\textsf{iid} to $0.1165$ under \textsf{point} to $0.0427$ under $2^\circ$ spatial
blocking, a factor of $3.0$. By Corollary~\ref{cor:protocol}, the price of every
prior on $D$ is forced to shrink with it. The apparent disappearance of constraint
cost under strict validation is thus a theorem, not a finding.

\begin{figure}[!t]
\centering
\includegraphics[width=\columnwidth]{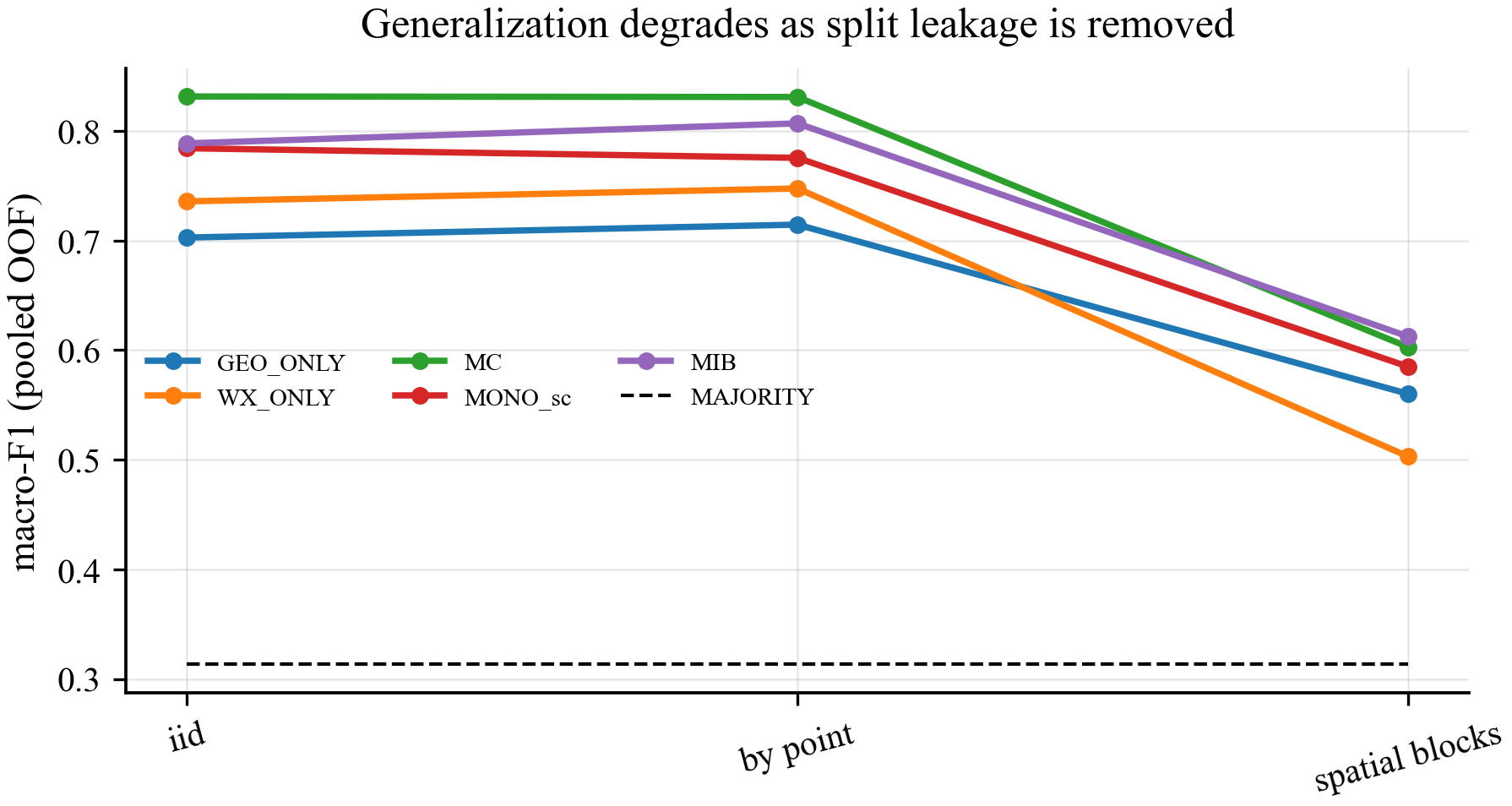}
\caption{Out-of-fold macro-$\Fone$ across the validation ladder, headline
configuration. The gap between \MC{} and \GEO{} is the certified ceiling
$\hat\head(D)$ on the cost of any monotone prior on the drivers
($0.1288\to0.1165\to0.0427$); the gap between \MC{} and \MONOsc{} is the measured
price ($0.0473\to0.0555\to0.0179$). Theorem~\ref{thm:main} requires the \MONOsc{}
curve never to fall below the \GEO{} curve, which it does not. Note that \WX{}
crosses below \GEO{} at the spatial rung: under blocking, coordinates alone beat
all four drivers together by $+0.0572$ $[+0.0252,+0.0814]$.}
\label{fig:ladder}
\end{figure}

\subsection{Price against certified ceiling}
Fig.~\ref{fig:cost} reports paired differences with $95\%$ cluster bootstrap
intervals against the three references of Table~\ref{tab:mapping}. The panel
against \GEO{} is the direct test of Corollary~\ref{cor:dominance}: every
constrained arm on $D\cup G$ must sit to the right of zero, and in the headline
configuration all fifteen such cells do, with margins from $+0.0029$ (\MIBfm,
spatial) to $+0.0941$ (\MIBfp, iid).

Reading Table~\ref{tab:main} as prescribed, the ratio $\hat\price/\hat\head$ is the
fraction of the drivers' value that the prior destroys: $0.367$ for \MONOsc{}
under \textsf{iid}, $0.476$ under \textsf{point}, $0.419$ under \textsf{spatial}.
The corresponding unshielded ratios are $0.821$, $0.822$ and $0.694$. That
contrast is the subject of Section~\ref{sec:shield}.

\begin{figure}[!t]
\centering
\includegraphics[width=\columnwidth]{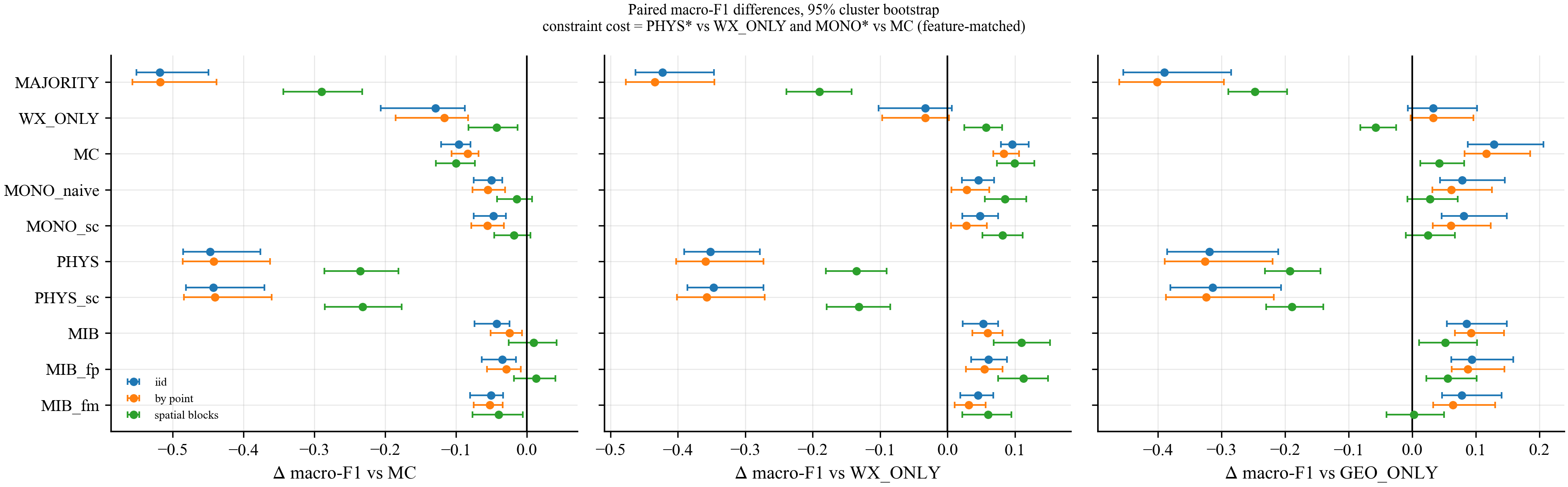}
\caption{Paired macro-$\Fone$ differences with $95\%$ cluster bootstrap intervals,
by protocol, headline configuration. \textsf{PHYS*} versus \WX{} and
\textsf{MONO*} versus \MC{} are the only feature-matched comparisons and hence the
only ones that measure a constraint cost (cf.\ Remark~\ref{rem:notmatched} for the
\textsf{MIB} arms); the panel against \GEO{} is the ablation dominance test of
Corollary~\ref{cor:dominance}, where every constrained arm must lie to the right of
zero. Note how the \textsf{PHYS*} points sit an order of magnitude further left
than the \textsf{MONO*} points against their respective unconstrained references
--- identical physics, different free feature set.}
\label{fig:cost}
\end{figure}

\begin{table*}[!t]
\renewcommand{\arraystretch}{1.15}
\caption{Headline configuration (\textsf{ctx}: $2^\circ$ blocks, gate auto/fold),
all eleven arms, all three protocols. $\hat\head$ is
$\MC-\GEO$ for arms on $D\cup G$ and $\WX-\MAJ$ for arms on $D$ alone
(Remark~\ref{rem:cons}). $\hat\price$ is the negated difference against the
feature-matched unconstrained reference. ``dominance'' is the sign test of
Corollary~\ref{cor:dominance} against the arm's own ablation and must be
$\ge0$. $\dagger$ marks $|\hat\price|<\hat s=0.0220$, i.e.\ cells that carry no
interpretation. $\ddagger$ marks the \textsf{MIB} arms, whose reference is not
feature-matched (Remark~\ref{rem:notmatched}). $\VR$ columns are end-to-end
violation rates against the arm's own declaration, the textbook prior, and the
screened prior.}
\label{tab:main}
\centering
\scriptsize
\begin{tabular}{@{}llccccc c ccc l@{}}
\toprule
Protocol & Model & $\Fone$ (OOF) & $\Fone$ fold $\pm$ sd &
$\hat\head$ & $\hat\price$ & $\hat\price/\hat\head$ & dominance &
$\VR_{\mathrm{own}}$ & $\VR_{\mathrm{nv}}$ & $\VR_{\mathrm{sc}}$ &
$\Fone$ per class \\
\midrule
\multirow{11}{*}{\textsf{iid}}
 & \MAJ    & 0.3132 & 0.3132 $\pm$ 0.0000 & ---    & ---    & ---    & ---      & n/a & n/a & n/a & 0.000 / 0.000 / 0.940 \\
 & \GEO    & 0.7028 & 0.7026 $\pm$ 0.0147 & ---    & ---    & ---    & ---      & n/a & n/a & n/a & 0.406 / 0.733 / 0.969 \\
 & \WX     & 0.7358 & 0.7356 $\pm$ 0.0116 & ---    & ---    & ---    & ---      & n/a & 0.4849 & 0.4953 & 0.597 / 0.646 / 0.964 \\
 & \MC     & 0.8316 & 0.8312 $\pm$ 0.0097 & 0.1288 & ---    & ---    & ---      & n/a & 0.4795 & 0.4957 & 0.700 / 0.816 / 0.979 \\
 & \MONOnv & 0.7815 & 0.7813 $\pm$ 0.0160 & 0.1288 & 0.0501 & 0.389 & $+$0.0787 & 0.0000 & 0.0000 & 0.1572 & 0.605 / 0.766 / 0.973 \\
 & \MONOsc & 0.7843 & 0.7843 $\pm$ 0.0095 & 0.1288 & 0.0473 & 0.367 & $+$0.0815 & 0.0000 & 0.2068 & 0.0000 & 0.602 / 0.776 / 0.974 \\
 & \PHYSnv & 0.3842 & 0.3842 $\pm$ 0.0100 & 0.4226 & 0.3516 & 0.832 & $+$0.0710 & 0.0000 & 0.0000 & 0.0852 & 0.000 / 0.211 / 0.941 \\
 & \PHYSsc & 0.3888 & 0.3885 $\pm$ 0.0138 & 0.4226 & 0.3470 & 0.821 & $+$0.0756 & 0.0000 & 0.3855 & 0.0000 & 0.000 / 0.225 / 0.942 \\
 & \MIB    & 0.7888 & 0.7881 $\pm$ 0.0219 & 0.1288 & 0.0428$^\ddagger$ & 0.332 & $+$0.0860 & 0.0000 & 0.2280 & 0.0000 & 0.586 / 0.803 / 0.978 \\
 & \MIBfp  & 0.7968 & 0.7969 $\pm$ 0.0095 & 0.1288 & 0.0347$^\ddagger$ & 0.269 & $+$0.0941 & 0.0000 & 0.2245 & 0.0000 & 0.616 / 0.797 / 0.977 \\
 & \MIBfm  & 0.7809 & 0.7808 $\pm$ 0.0161 & 0.1288 & 0.0507$^\ddagger$ & 0.394 & $+$0.0781 & 0.0000 & 0.2035 & 0.0000 & 0.593 / 0.776 / 0.974 \\
\midrule
\multirow{11}{*}{\textsf{point}}
 & \MAJ    & 0.3132 & 0.3132 $\pm$ 0.0004 & ---    & ---    & ---    & ---      & n/a & n/a & n/a & 0.000 / 0.000 / 0.940 \\
 & \GEO    & 0.7146 & 0.7143 $\pm$ 0.0203 & ---    & ---    & ---    & ---      & n/a & n/a & n/a & 0.433 / 0.741 / 0.970 \\
 & \WX     & 0.7476 & 0.7471 $\pm$ 0.0132 & ---    & ---    & ---    & ---      & n/a & 0.4922 & 0.4928 & 0.610 / 0.667 / 0.966 \\
 & \MC     & 0.8311 & 0.8309 $\pm$ 0.0196 & 0.1165 & ---    & ---    & ---      & n/a & 0.4799 & 0.4868 & 0.691 / 0.823 / 0.980 \\
 & \MONOnv & 0.7761 & 0.7755 $\pm$ 0.0175 & 0.1165 & 0.0549 & 0.471 & $+$0.0616 & 0.0000 & 0.0000 & 0.0494 & 0.586 / 0.769 / 0.973 \\
 & \MONOsc & 0.7755 & 0.7750 $\pm$ 0.0184 & 0.1165 & 0.0555 & 0.476 & $+$0.0609 & 0.0000 & 0.0747 & 0.0000 & 0.581 / 0.771 / 0.974 \\
 & \PHYSnv & 0.3888 & 0.3888 $\pm$ 0.0086 & 0.4344 & 0.3588 & 0.826 & $+$0.0756 & 0.0000 & 0.0000 & 0.0326 & 0.000 / 0.225 / 0.942 \\
 & \PHYSsc & 0.3907 & 0.3907 $\pm$ 0.0087 & 0.4344 & 0.3569 & 0.822 & $+$0.0775 & 0.0000 & 0.1661 & 0.0000 & 0.000 / 0.230 / 0.942 \\
 & \MIB    & 0.8070 & 0.8066 $\pm$ 0.0098 & 0.1165 & 0.0241$^\ddagger$ & 0.207 & $+$0.0924 & 0.0000 & 0.0880 & 0.0000 & 0.637 / 0.806 / 0.978 \\
 & \MIBfp  & 0.8023 & 0.8021 $\pm$ 0.0056 & 0.1165 & 0.0288$^\ddagger$ & 0.247 & $+$0.0877 & 0.0000 & 0.0912 & 0.0000 & 0.634 / 0.795 / 0.977 \\
 & \MIBfm  & 0.7790 & 0.7786 $\pm$ 0.0115 & 0.1165 & 0.0521$^\ddagger$ & 0.447 & $+$0.0644 & 0.0000 & 0.0826 & 0.0000 & 0.592 / 0.772 / 0.974 \\
\midrule
\multirow{11}{*}{\textsf{spatial}}
 & \MAJ    & 0.3132 & 0.3109 $\pm$ 0.0069 & ---    & ---    & ---    & ---      & n/a & n/a & n/a & 0.000 / 0.000 / 0.940 \\
 & \GEO    & 0.5602 & 0.5544 $\pm$ 0.0617 & ---    & ---    & ---    & ---      & n/a & n/a & n/a & 0.082 / 0.640 / 0.958 \\
 & \WX     & 0.5030 & 0.5045 $\pm$ 0.0628 & ---    & ---    & ---    & ---      & n/a & 0.4904 & 0.4888 & 0.130 / 0.433 / 0.946 \\
 & \MC     & 0.6029 & 0.6014 $\pm$ 0.0721 & 0.0427 & ---    & ---    & ---      & n/a & 0.4825 & 0.4968 & 0.208 / 0.640 / 0.961 \\
 & \MONOnv & 0.5886 & 0.5883 $\pm$ 0.0635 & 0.0427 & 0.0143$^\dagger$ & 0.335 & $+$0.0284 & 0.0000 & 0.0000 & 0.1904 & 0.194 / 0.613 / 0.959 \\
 & \MONOsc & 0.5850 & 0.5856 $\pm$ 0.0639 & 0.0427 & 0.0179$^\dagger$ & 0.419 & $+$0.0248 & 0.0000 & 0.2650 & 0.0000 & 0.183 / 0.613 / 0.959 \\
 & \PHYSnv & 0.3679 & 0.3643 $\pm$ 0.0271 & 0.1898 & 0.1351 & 0.712 & $+$0.0547 & 0.0000 & 0.0000 & 0.0638 & 0.006 / 0.161 / 0.937 \\
 & \PHYSsc & 0.3713 & 0.3692 $\pm$ 0.0326 & 0.1898 & 0.1317 & 0.694 & $+$0.0581 & 0.0000 & 0.3462 & 0.0000 & 0.000 / 0.174 / 0.939 \\
 & \MIB    & 0.6125 & 0.6105 $\pm$ 0.0765 & 0.0427 & $-$0.0096$^{\dagger\ddagger}$ & $-$0.225 & $+$0.0523 & 0.0000 & 0.2885 & 0.0000 & 0.192 / 0.681 / 0.964 \\
 & \MIBfp  & 0.6159 & 0.6102 $\pm$ 0.0886 & 0.0427 & $-$0.0130$^{\dagger\ddagger}$ & $-$0.305 & $+$0.0557 & 0.0000 & 0.2863 & 0.0000 & 0.204 / 0.679 / 0.964 \\
 & \MIBfm  & 0.5630 & 0.5626 $\pm$ 0.0516 & 0.0427 & 0.0399$^\ddagger$ & 0.934 & $+$0.0029 & 0.0000 & 0.2494 & 0.0000 & 0.127 / 0.604 / 0.958 \\
\bottomrule
\end{tabular}
\end{table*}

\subsection{Shielding, measured}
\label{sec:shield}
The shield's strength is $\Fone(\GEO)=0.7028$ under \textsf{iid} against a
majority baseline of $0.3132$, and $0.5602$ under spatial blocking. Expressed as a
fraction of the full model it rises with protocol severity: coordinates alone
recover $84.5\%$ of \MC{} under \textsf{iid} and $92.9\%$ under blocking. Under
blocking they also \emph{beat} the four drivers together by
$+0.0572$ $[+0.0252,+0.0814]$.

The consequence for constraint accounting is stark. The same screened monotone
prior costs
\begin{align*}
0.0473 \quad &\text{shielded:\ \ \MONOsc{} versus \MC{}},\\
0.3470 \quad &\text{unshielded:\ \PHYSsc{} versus \WX{}},
\end{align*}
a ratio of $7.34$ under \textsf{iid}; the ratio is $6.43$ under \textsf{point} and
$7.36$ under \textsf{spatial}, and $7.02$ for the textbook prior under
\textsf{iid}. Nothing about the physics differs between the two arms; only the
free feature set does. Nor is the effect a pure rescaling by $\hat\head$: the
normalized price $\hat\price/\hat\head$ also more than doubles, from $0.367$
shielded to $0.821$ unshielded. Under Proposition~\ref{prop:sharp} the unshielded
ratio of $0.82$ is close to the anti-alignment regime, i.e.\ on the drivers alone
the monotone prior destroys most of what the drivers are worth --- a conclusion
the shielded arm hides completely. Had we reported the shielded configuration
alone --- the natural choice, since coordinates improve absolute performance ---
we would have concluded that the prior is nearly free. This is
Proposition~\ref{prop:shield} in numbers, and it is the single most important
reason we recommend reporting the free feature set alongside every
constraint-cost figure.

The block-size sweep isolates the mechanism, and inverts the naive expectation.
Coarsening blocks might be expected to weaken the coordinate shield and so
\emph{raise} $\hat\head(D)$; instead $\Fone(\GEO)$ moves only from $0.5592$ to
$0.5168$ while $\Fone(\MC)$ collapses from $0.6534$ to $0.5218$
(Table~\ref{tab:grid_f1}, \textsf{spatial} rows), so
\[
\hat\head(D):\quad 0.0942 \;\to\; 0.0427 \;\to\; 0.0184 \;\to\; 0.0050 .
\]
Coordinates transfer across coarse block boundaries far better than the drivers
do: $\Fone(\WX)$ falls from $0.5686$ to $0.4410$ over the same sweep. What
Proposition~\ref{prop:shield} guarantees is monotonicity of $\head$ in the
\emph{feature sets}, not in the protocol; Corollary~\ref{cor:protocol} is what
governs the sweep, and it bites hard. The measured price tracks the ceiling down
and then inverts: $0.0456$, $0.0179$, $0.0229$, $-0.0139$ for \MONOsc{}. At
$10^\circ$ the ceiling $0.0050$ is a fifth of $\hat s$, so the last two entries
are not measurements of anything.

\begin{table*}[!t]
\renewcommand{\arraystretch}{1.12}
\caption{Complete out-of-fold macro-$\Fone$ across the grid. \textsf{ctx} is the
headline configuration; \textsf{b2} and \textsf{glob} reproduce it bytewise and
are not repeated. Rows marked $\ast$ have an empty exogeneity gate, so
$\textsf{MIB*}\equiv\MONOsc$ by Lemma~\ref{lem:degen} (verified bytewise).
$\hat\head_{G}=\MC-\GEO$ is the certified ceiling for the shielded arms and
$\hat\head_{M}=\WX-\MAJ$ for the unshielded arms. Because block size also sets the
bootstrap clustering, the \textsf{iid} and \textsf{point} rows of
\textsf{b1}/\textsf{b5}/\textsf{b10} reproduce \textsf{ctx} exactly in point
estimate and differ only in interval width.}
\label{tab:grid_f1}
\centering
\scriptsize
\begin{tabular}{@{}ll ccc cc cc ccc cc@{}}
\toprule
Config & Protocol & \MAJ & \GEO & \WX & \MC & \MONOnv & \MONOsc & \PHYSnv & \PHYSsc & \MIB & \MIBfp & \MIBfm & $\hat\head_{G}$ / $\hat\head_{M}$ \\
\midrule
\multirow{3}{*}{\textsf{ctx}}
 & \textsf{iid}     & 0.3132 & 0.7028 & 0.7358 & 0.8316 & 0.7815 & 0.7843 & 0.3842 & 0.3888 & 0.7888 & 0.7968 & 0.7809 & 0.1288 / 0.4226 \\
 & \textsf{point}   & 0.3132 & 0.7146 & 0.7476 & 0.8311 & 0.7761 & 0.7755 & 0.3888 & 0.3907 & 0.8070 & 0.8023 & 0.7790 & 0.1165 / 0.4344 \\
 & \textsf{spatial} & 0.3132 & 0.5602 & 0.5030 & 0.6029 & 0.5886 & 0.5850 & 0.3679 & 0.3713 & 0.6125 & 0.6159 & 0.5630 & 0.0427 / 0.1898 \\
\midrule
\multirow{3}{*}{\textsf{b1}}
 & \textsf{iid}     & 0.3132 & 0.7028 & 0.7358 & 0.8316 & 0.7815 & 0.7843 & 0.3842 & 0.3888 & 0.7888 & 0.7968 & 0.7809 & 0.1288 / 0.4226 \\
 & \textsf{point}   & 0.3132 & 0.7146 & 0.7476 & 0.8311 & 0.7761 & 0.7755 & 0.3888 & 0.3907 & 0.8070 & 0.8023 & 0.7790 & 0.1165 / 0.4344 \\
 & \textsf{spatial} & 0.3132 & 0.5592 & 0.5686 & 0.6534 & 0.6061 & 0.6078 & 0.3711 & 0.3996 & 0.6182 & 0.6381 & 0.5982 & 0.0942 / 0.2554 \\
\midrule
\multirow{3}{*}{\textsf{b5}}
 & \textsf{iid}     & 0.3132 & 0.7028 & 0.7358 & 0.8316 & 0.7815 & 0.7843 & 0.3842 & 0.3888 & 0.7888 & 0.7968 & 0.7809 & 0.1288 / 0.4226 \\
 & \textsf{point}   & 0.3132 & 0.7146 & 0.7476 & 0.8311 & 0.7761 & 0.7755 & 0.3888 & 0.3907 & 0.8070 & 0.8023 & 0.7790 & 0.1165 / 0.4344 \\
 & \textsf{spatial} & 0.3132 & 0.5670 & 0.4975 & 0.5854 & 0.5846 & 0.5625 & 0.3443 & 0.3865 & 0.5696 & 0.5847 & 0.5516 & 0.0184 / 0.1844 \\
\midrule
\multirow{3}{*}{\textsf{b10}}
 & \textsf{iid}     & 0.3132 & 0.7028 & 0.7358 & 0.8316 & 0.7815 & 0.7843 & 0.3842 & 0.3888 & 0.7888 & 0.7968 & 0.7809 & 0.1288 / 0.4226 \\
 & \textsf{point}   & 0.3132 & 0.7146 & 0.7476 & 0.8311 & 0.7761 & 0.7755 & 0.3888 & 0.3907 & 0.8070 & 0.8023 & 0.7790 & 0.1165 / 0.4344 \\
 & \textsf{spatial} & 0.3132 & 0.5168 & 0.4410 & 0.5218 & 0.5419 & 0.5357 & 0.3439 & 0.3569 & 0.5488 & 0.5411 & 0.5460 & 0.0050 / 0.1279 \\
\midrule
\multirow{3}{*}{\textsf{closed}$^\ast$}
 & \textsf{iid}     & 0.3132 & 0.7028 & 0.7358 & 0.8316 & 0.7815 & 0.7843 & 0.3842 & 0.3888 & 0.7843 & 0.7843 & 0.7843 & 0.1288 / 0.4226 \\
 & \textsf{point}   & 0.3132 & 0.7146 & 0.7476 & 0.8311 & 0.7761 & 0.7755 & 0.3888 & 0.3907 & 0.7755 & 0.7755 & 0.7755 & 0.1165 / 0.4344 \\
 & \textsf{spatial} & 0.3132 & 0.5602 & 0.5030 & 0.6029 & 0.5886 & 0.5850 & 0.3679 & 0.3713 & 0.5850 & 0.5850 & 0.5850 & 0.0427 / 0.1898 \\
\midrule
\multirow{3}{*}{\textsf{drv}$^\ast$}
 & \textsf{iid}     & 0.3132 & 0.7028 & 0.7807 & 0.8427 & 0.7956 & 0.7996 & 0.4208 & 0.4229 & 0.7996 & 0.7996 & 0.7996 & 0.1399 / 0.4675 \\
 & \textsf{point}   & 0.3132 & 0.7146 & 0.7850 & 0.8443 & 0.7965 & 0.7991 & 0.4282 & 0.4293 & 0.7991 & 0.7991 & 0.7991 & 0.1297 / 0.4718 \\
 & \textsf{spatial} & 0.3132 & 0.5602 & 0.5371 & 0.6207 & 0.6018 & 0.6136 & 0.4133 & 0.4263 & 0.6136 & 0.6136 & 0.6136 & 0.0606 / 0.2239 \\
\midrule
\multirow{3}{*}{\textsf{k5} ($K{=}4$)}
 & \textsf{iid}     & 0.2043 & 0.6701 & 0.6763 & 0.7904 & 0.7361 & 0.7361 & 0.2979 & 0.2979 & 0.7549 & 0.7448 & 0.7385 & 0.1203 / 0.4720 \\
 & \textsf{point}   & 0.2043 & 0.6684 & 0.6887 & 0.7987 & 0.7288 & 0.7288 & 0.2956 & 0.2956 & 0.7616 & 0.7428 & 0.7355 & 0.1303 / 0.4844 \\
 & \textsf{spatial} & 0.2043 & 0.4608 & 0.4112 & 0.5262 & 0.5058 & 0.4912 & 0.2935 & 0.2912 & 0.5148 & 0.5071 & 0.5074 & 0.0655 / 0.2069 \\
\bottomrule
\end{tabular}
\end{table*}

\subsection{Feature-set and label-set perturbations}
Two configurations perturb $S$ and $\Yc$ rather than the split, and both behave as
Proposition~\ref{prop:shield} requires.

Promoting insolation from auxiliary channel to fifth driver (\textsf{drv})
enlarges $S$ from four to five features. The ceiling rises accordingly, from
$0.1288$ to $0.1399$ under \textsf{iid} and from $0.4226$ to $0.4675$ for the
unshielded arms, while the price rises less ($0.0473\to0.0431$ shielded,
$0.3470\to0.3578$ unshielded), so the normalized price falls to $0.308$ shielded
and $0.765$ unshielded. The channel that the exogeneity gate classified as
auxiliary is in fact the single best-identified driver in the whole study: its
sign is positive in $15/15$ folds, agrees with the textbook prior in $15/15$, and
its bootstrap interval excludes zero in $12/15$ --- against $0$--$6/15$ for the
four meteorological drivers. Exogeneity ($R^2=0.2646$ on the driver design) is a
statement about redundancy, not about predictive value, and the gate is doing the
job it was designed for rather than the job of feature selection.

Refining the label set to $K=4$ (\textsf{k5}) leaves the ceiling essentially
unchanged ($0.1203$, $0.1303$, $0.0655$ across the ladder) but raises the price
sharply: $\hat\price/\hat\head$ for \MONOsc{} becomes $0.451$, $0.537$ and $0.534$
versus $0.367$, $0.476$ and $0.419$ at $K=3$. Under
Proposition~\ref{prop:sharp} this is the expected reading: at finer label
resolution the monotone prior is a worse approximation to the mechanism.
Consistently, this is also the configuration in which the drivers' signs become
identifiable (Table~\ref{tab:signs}).

\subsection{Sign identification, and why its stakes are bounded}
Fig.~\ref{fig:signs} shows fold-wise mean partial effects with clustered
intervals; Table~\ref{tab:signs} aggregates. In the headline configuration the
screened sign agrees with the textbook sign in $53$ of $60$ driver-folds
($88.3\%$) and the interval excludes zero in only $13$ of $60$ ($21.7\%$).
Temperature and relative humidity agree in $15/15$ folds each; wind in $14/15$;
precipitation in only $9/15$, with the screened sign coming out \emph{positive}
--- against the textbook $-1$ --- in $6$ of $15$ folds while being identified in
just $2$. Precipitation is the one driver on which a practitioner would face a
genuine textbook-versus-data dispute, and it is precisely the driver the data
cannot resolve.

By Corollary~\ref{cor:spread} the accuracy at stake in that dispute is bounded by
$\hat\head=0.0427$ under spatial blocking and $0.1288$ under \textsf{iid},
regardless of how it is resolved --- and by $0.0050$ at $10^\circ$ blocks, where it
is unresolvable in principle. Note that this is a bound on \emph{accuracy}
consequences only; a misspecified sign can still be arbitrarily damaging out of
distribution, which no in-distribution metric detects.

Table~\ref{tab:signs} also exposes an artefact that deserves its own warning.
``Identified'' counts folds whose $95\%$ clustered interval excludes zero, and the
clusters are the spatial blocks. Refining blocks therefore manufactures
identification: temperature is identified in $9/15$ folds at $1^\circ$
($1080$ clusters), $5/15$ at $2^\circ$ ($429$), $4/15$ at $5^\circ$ ($115$) and
$6/15$ at $10^\circ$ ($40$). Sign identifiability is thus partly a property of the
analyst's blocking choice, not of the physics. What is \emph{not} an artefact is
the effect of label resolution: at $K=4$ every driver's mean effect grows by a
factor of $2$--$40$ and precipitation becomes identified in $15/15$ folds with
mean effect $-0.01144$ against $-0.00028$ at $K=3$.

\begin{figure}[!t]
\centering
\includegraphics[width=\columnwidth]{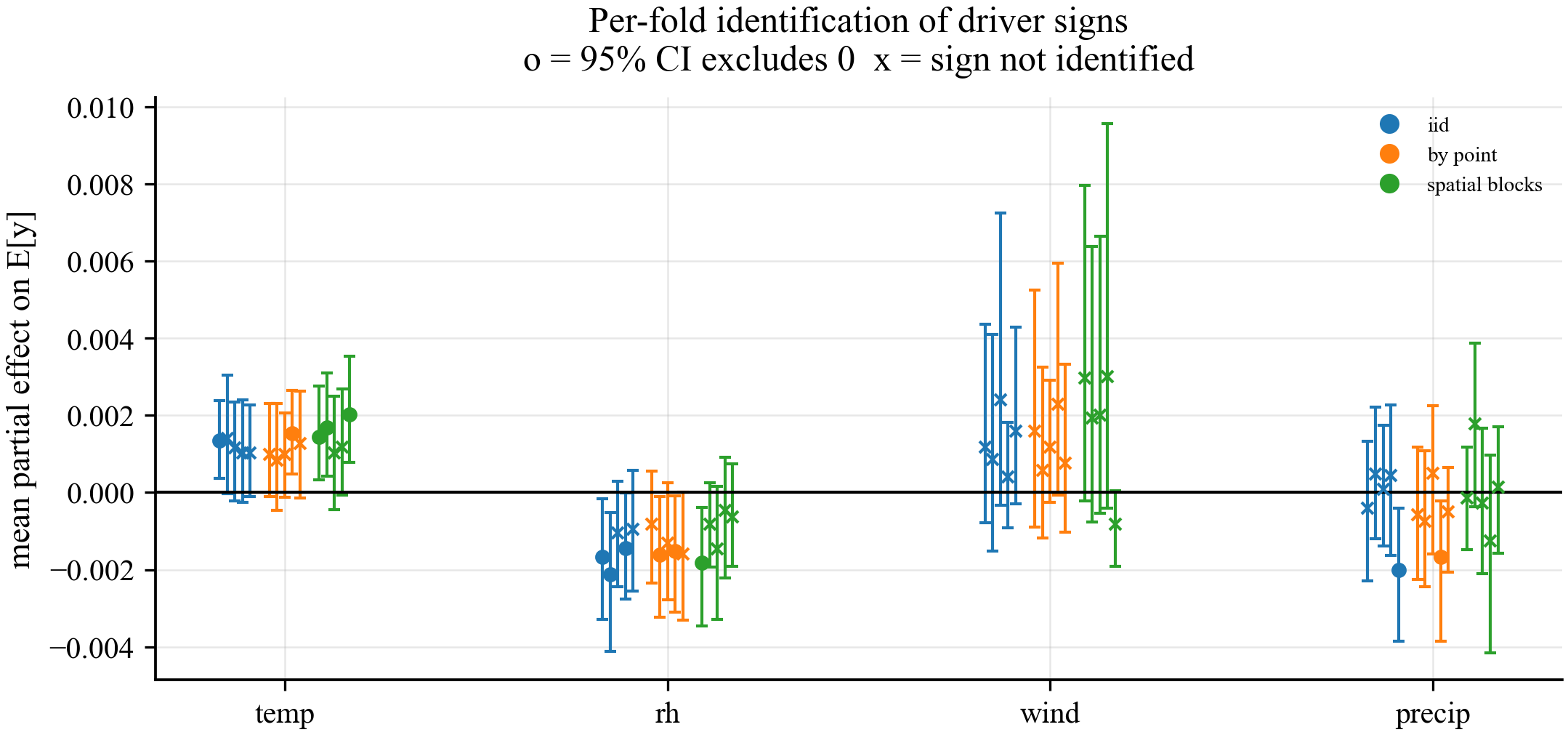}
\caption{Fold-wise identification of driver signs, headline configuration: mean
central-difference partial effect of the unconstrained model with $95\%$ clustered
intervals, three protocols $\times$ five folds per driver. Circles mark folds
where the interval excludes zero ($13$ of $60$); crosses mark non-identified signs.
Temperature and wind sit above zero, relative humidity below; precipitation
straddles zero, which is why its screened sign contradicts the textbook prior in
$6$ of $15$ folds. A non-identified sign implies little headroom and hence, by
Corollary~\ref{cor:spread}, little accuracy at stake in choosing it.}
\label{fig:signs}
\end{figure}

\begin{table}[!t]
\renewcommand{\arraystretch}{1.15}
\caption{Sign screening across $15$ training folds ($3$ protocols $\times$ $5$
folds) per configuration. ``agree'' counts folds where the screened sign matches
the textbook prior; ``ident.'' counts folds whose $95\%$ clustered interval
excludes zero; ``$n_+$'' counts folds with a positive screened sign; ``eff.'' is
the mean partial effect on $E[y]$. Textbook signs are
$\textit{temp}{+}$, $\textit{rh}{-}$, $\textit{wind}{+}$, $\textit{precip}{-}$,
$\textit{solar}{+}$.}
\label{tab:signs}
\centering
\footnotesize
\begin{tabular}{@{}llcccr@{}}
\toprule
Config & Driver & agree & ident. & $n_+$ & eff. \\
\midrule
\multirow{4}{*}{\textsf{ctx} (headline)}
 & \textit{temp}   & 15/15 & 5/15  & 15/15 & $+0.00126$ \\
 & \textit{rh}     & 15/15 & 6/15  & 0/15  & $-0.00129$ \\
 & \textit{wind}   & 14/15 & 0/15  & 14/15 & $+0.00146$ \\
 & \textit{precip} & 9/15  & 2/15  & 6/15  & $-0.00028$ \\
\midrule
\multirow{4}{*}{\textsf{b1} ($1^\circ$)}
 & \textit{temp}   & 15/15 & 9/15  & 15/15 & $+0.00129$ \\
 & \textit{rh}     & 15/15 & 8/15  & 0/15  & $-0.00130$ \\
 & \textit{wind}   & 14/15 & 3/15  & 14/15 & $+0.00153$ \\
 & \textit{precip} & 9/15  & 4/15  & 6/15  & $-0.00048$ \\
\midrule
\multirow{4}{*}{\textsf{b5} ($5^\circ$)}
 & \textit{temp}   & 15/15 & 4/15  & 15/15 & $+0.00124$ \\
 & \textit{rh}     & 14/15 & 2/15  & 1/15  & $-0.00128$ \\
 & \textit{wind}   & 14/15 & 0/15  & 14/15 & $+0.00135$ \\
 & \textit{precip} & 8/15  & 2/15  & 7/15  & $-0.00022$ \\
\midrule
\multirow{4}{*}{\textsf{b10} ($10^\circ$)}
 & \textit{temp}   & 15/15 & 6/15  & 15/15 & $+0.00121$ \\
 & \textit{rh}     & 15/15 & 2/15  & 0/15  & $-0.00133$ \\
 & \textit{wind}   & 15/15 & 0/15  & 15/15 & $+0.00123$ \\
 & \textit{precip} & 8/15  & 1/15  & 7/15  & $-0.00032$ \\
\midrule
\multirow{5}{*}{\textsf{drv} ($|S|{=}5$)}
 & \textit{temp}   & 15/15 & 2/15  & 15/15 & $+0.00049$ \\
 & \textit{rh}     & 15/15 & 1/15  & 0/15  & $-0.00085$ \\
 & \textit{wind}   & 14/15 & 1/15  & 14/15 & $+0.00105$ \\
 & \textit{precip} & 8/15  & 1/15  & 7/15  & $+0.00013$ \\
 & \textit{solar}  & 15/15 & 12/15 & 15/15 & $+0.00166$ \\
\midrule
\multirow{4}{*}{\textsf{k5} ($K{=}4$)}
 & \textit{temp}   & 14/15 & 6/15  & 14/15 & $+0.00336$ \\
 & \textit{rh}     & 15/15 & 12/15 & 0/15  & $-0.00274$ \\
 & \textit{wind}   & 15/15 & 13/15 & 15/15 & $+0.00501$ \\
 & \textit{precip} & 15/15 & 15/15 & 0/15  & $-0.01144$ \\
\bottomrule
\end{tabular}
\end{table}

\subsection{Compliance is not price, empirically}
\label{sec:compliance}
Constrained arms satisfy their own declaration exactly, $\VR_{\text{own}}=0.0000$
in every arm, protocol and configuration; the harness asserts this and aborts
otherwise, so it is a plumbing test rather than a result
(Table~\ref{tab:main}, ninth column). The informative comparison is between the
unconstrained arms' violation rates and the measured prices.

The unconstrained \MC{} model violates the textbook prior on the drivers at rate
$0.4795$ (\textsf{iid}), $0.4799$ (\textsf{point}) and $0.4825$
(\textsf{spatial}), computed over $40\,873$--$45\,071$ active partial-effect
evaluations; against the screened prior the rates are $0.4957$, $0.4868$ and
$0.4968$. Per driver under blocking they are $0.432$ (\textit{temp}), $0.492$
(\textit{rh}), $0.548$ (\textit{wind}) and $0.463$ (\textit{precip}), i.e.\
indistinguishable from the $0.5$ of a coin flip (Fig.~\ref{fig:prior}). The
measured price of enforcing that same prior is $0.0143$--$0.0555$. A rate of
$0.48$ would ordinarily be reported as ``the model has not learned the physics'';
a price of $0.02$ would ordinarily be reported as ``the physics is free.'' Both
statements hold simultaneously here, exactly as Proposition~\ref{prop:vr}
predicts, because the drivers are largely redundant given coordinates: the sign of
a partial effect is unconstrained wherever the effect does not matter.

The one place where the rates are informative is the comparison across
declarations. \MONOsc{} violates the \emph{textbook} prior at $0.2068$
(\textsf{iid}) and $0.2650$ (\textsf{spatial}) --- it must, since its own
declaration differs from the textbook one in $6$ of $15$ folds --- while
\MONOnv{} violates the \emph{screened} prior at $0.1572$ and $0.1904$. \PHYSsc{},
which has no shield to hide behind, disagrees with the textbook prior most
strongly of all, at $0.3855$ and $0.3462$. Those numbers quantify how far apart
the two priors are in deployment, and they are the only $\VR$ figures in this
paper that are comparable across arms.

Fig.~\ref{fig:leak} additionally distinguishes the \emph{certificate} from the
\emph{deployed pipeline}. When the auxiliary channel is admitted the residualizer
sits upstream of the constrained learner, so a nonzero $\VR_{\text{e2e}}$ against
the declared prior would reveal that the certificate does not survive
preprocessing. \textbf{No such leakage was observed}: $\VR_{\text{own}}=0.0000$
end-to-end for all three \textsf{MIB} arms in all three protocols of all five
gate-open configurations. Residualizing the channel against a quadratic design in
$D$ does not reintroduce a sign violation on $D$.

\begin{figure}[!t]
\centering
\includegraphics[width=\columnwidth]{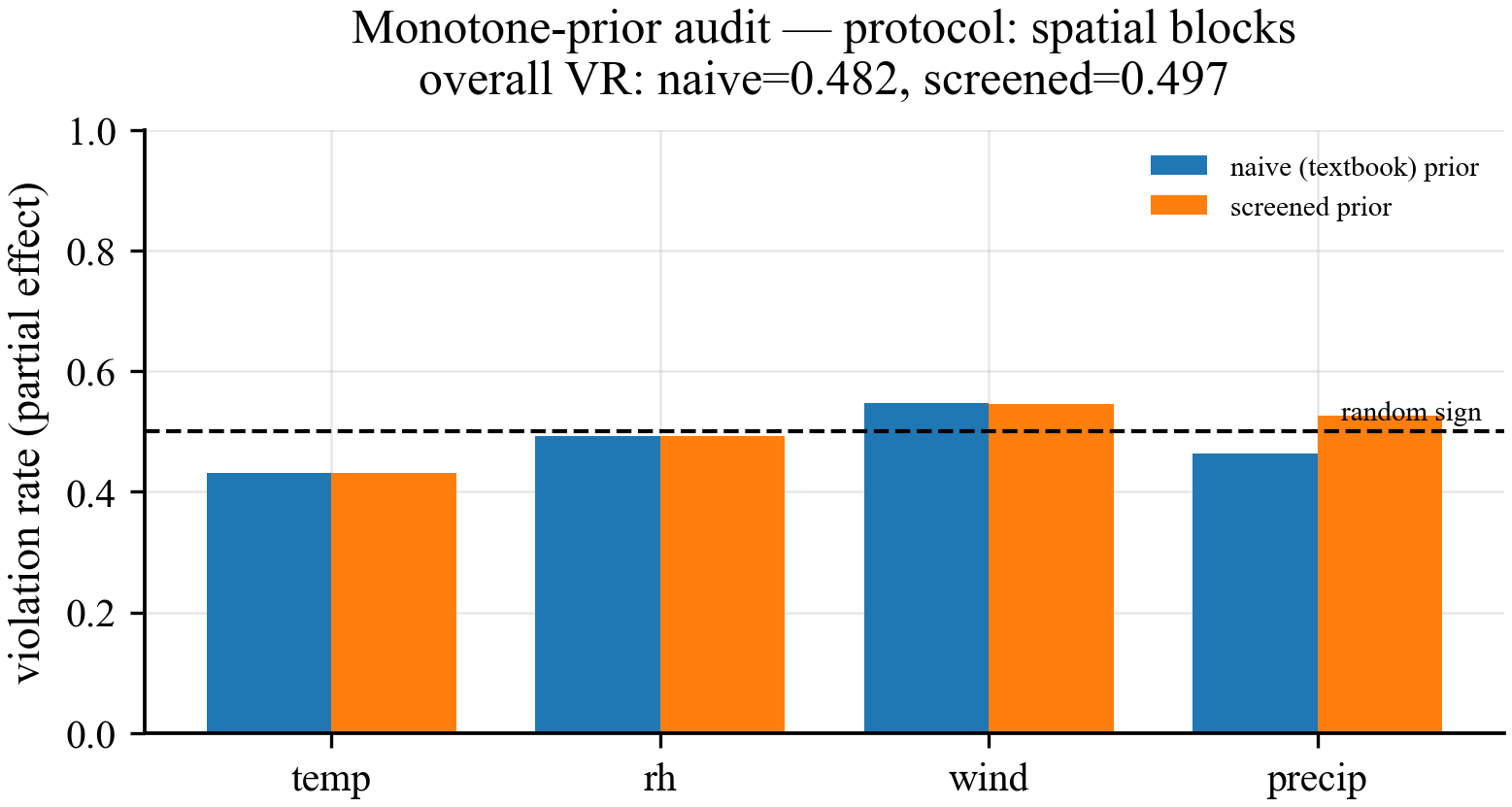}
\caption{Violation rate of the \emph{unconstrained} \MC{} model against the
textbook and the screened prior, per driver, under $2^\circ$ spatial blocking
(overall $0.4825$ naive, $0.4968$ screened, over $45\,071$ active evaluations).
Every bar sits near the dashed $0.5$ line, indicating a partial effect whose sign
is essentially arbitrary in the unconstrained fit. The only visible gap is
\textit{precip}, whose screened rate ($0.527$) exceeds its naive rate ($0.463$)
because the screened sign flips in some folds. By Proposition~\ref{prop:vr} these
rates carry no implication for the price of enforcing the prior, which is
$0.0179$ in the same cell.}
\label{fig:prior}
\end{figure}

\begin{figure}[!t]
\centering
\includegraphics[width=\columnwidth]{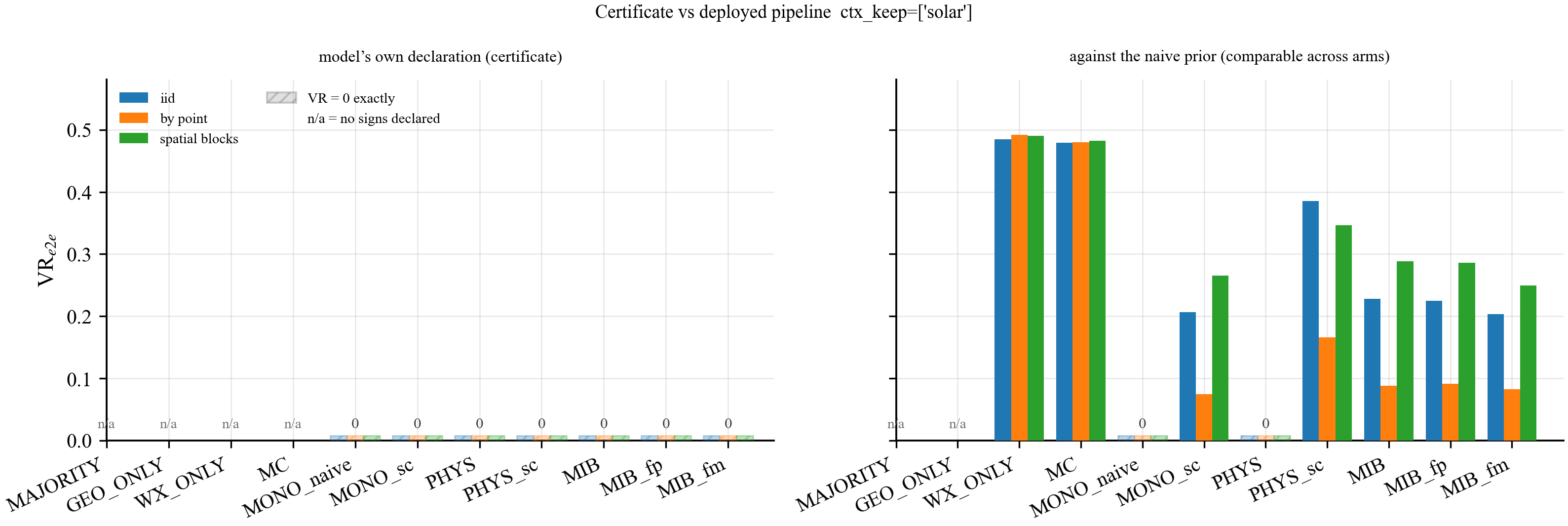}
\caption{Certificate versus deployed pipeline, headline configuration
($c=$ insolation admitted). Left: violations of each model's own declared sign
pattern, evaluated end-to-end through the residualizer. \emph{Every bar in this
panel is exactly $0.0000$} --- the constrained arms are feasible by construction
of the hard split constraints and remain feasible after preprocessing --- and the
four unconstrained arms declare nothing, hence \textsf{n/a}. The panel is
therefore flat by design; it is a passed plumbing test, not a missing plot.
Right: violations against the common textbook prior, which is comparable across
all arms. \MONOnv{} and \PHYSnv{} read $0.0000$ there because the textbook prior
\emph{is} their own declaration.}
\label{fig:leak}
\end{figure}

\subsection{The auxiliary channel: gate, and bounded sign sensitivity}
\label{sec:ctxarm}
Of the five candidate auxiliary channels the harness looks for, only insolation is
present in this corpus; dewpoint, apparent temperature, pressure and cloud cover
are absent from the join, and calendar harmonics are unavailable because no
parsable date column exists. Insolation passes the exogeneity gate decisively
($R^2=0.2646$ against the threshold $0.90$; residual relative standard deviation
$0.8576$ against the threshold $0.10$) and was admitted in all $15$ training folds
of every gate-auto configuration. Two regimes must then be distinguished, and only
one of them is informative.

In the gate-closed regime (\textsf{closed}, and \textsf{drv} where insolation has
been promoted to a driver so no candidate remains), \MIB{}, \MIBfp{} and \MIBfm{}
are \emph{identical} to \MONOsc{}, and the sign-sensitivity spread is exactly
$0.0000$ in all three protocols. This is the predicted degeneracy of
Lemma~\ref{lem:degen}, verified bytewise on out-of-fold predictions by check C5;
it is a successful test of the implementation and simultaneously a null result
about the auxiliary channel.

In the gate-open regime, Corollary~\ref{cor:spread} predicts that the spread over
the three sign choices cannot exceed $\Fone(\MIB)-\Fone(\MONOsc)$.
Table~\ref{tab:ctx} reports both quantities. The bound holds in $8$ of the $15$
gate-open cells and is \emph{violated} in $7$, by up to $0.0295$. The violations
are not mysterious: each one decomposes exactly into the two elementary nesting
inversions of Proposition~\ref{prop:ceilings} that produce it,
\begin{align*}
\text{excess}\;=\;
&\underbrace{\bigl(\Fone(\MIBfp)-\Fone(\MIB)\bigr)_+}_{\text{R4}}\\
+\;&\underbrace{\bigl(\Fone(\MONOsc)-\Fone(\MIBfm)\bigr)_+}_{\text{R3}},
\end{align*}
to four decimal places in all seven cells --- for instance
$0.0295=0.0199+0.0096$ at \textsf{b1}/\textsf{spatial} and
$0.0254=0.0034+0.0220$ at \textsf{ctx}/\textsf{spatial}. Each summand is an
inversion of a certified ordering and is therefore already accounted for in
$\hat s$; the composite excess must not be added to $\hat s$ separately, on pain
of double counting. Interpreted through Corollary~\ref{cor:slack}, the finding is
that the sign imposed on an auxiliary channel changes the outcome by more than the
channel is worth, which is only possible because the pipeline's resolution is
coarser than the channel's value: $\Fone(\MIB)-\Fone(\MONOsc)=0.0045$ under
\textsf{iid} and $0.0275$ under blocking, against $\hat s=0.0220$. The honest
conclusion about the auxiliary channel in this corpus is that its contribution is
at or below the noise floor, and that its sign therefore cannot be chosen on
accuracy grounds.

\begin{table}[!t]
\renewcommand{\arraystretch}{1.15}
\caption{Auxiliary-channel arms in the five gate-open configurations. ``spread''
is $\max-\min$ over $\{\MIB,\MIBfp,\MIBfm\}$; ``bound'' is
$\Fone(\MIB)-\Fone(\MONOsc)$ as required by \eqref{eq:ctxspread}. Cells where the
spread exceeds the bound are marked $\times$; the excess equals the sum of the
R4 and R3 inversions listed in Table~\ref{tab:slack}. Gate-closed configurations
(\textsf{closed}, \textsf{drv}) have spread $=$ bound $=0.0000$ exactly and are
omitted.}
\label{tab:ctx}
\centering
\footnotesize
\begin{tabular}{@{}llccc@{}}
\toprule
Config & Protocol & spread & bound & excess \\
\midrule
\textsf{ctx}  & \textsf{iid}     & 0.0160 & 0.0045 & $\times$ 0.0114 \\
\textsf{ctx}  & \textsf{point}   & 0.0280 & 0.0315 & --- \\
\textsf{ctx}  & \textsf{spatial} & 0.0529 & 0.0275 & $\times$ 0.0254 \\
\textsf{b1}   & \textsf{iid}     & 0.0160 & 0.0045 & $\times$ 0.0114 \\
\textsf{b1}   & \textsf{point}   & 0.0280 & 0.0315 & --- \\
\textsf{b1}   & \textsf{spatial} & 0.0399 & 0.0104 & $\times$ 0.0295 \\
\textsf{b5}   & \textsf{iid}     & 0.0160 & 0.0045 & $\times$ 0.0114 \\
\textsf{b5}   & \textsf{point}   & 0.0280 & 0.0315 & --- \\
\textsf{b5}   & \textsf{spatial} & 0.0331 & 0.0071 & $\times$ 0.0260 \\
\textsf{b10}  & \textsf{iid}     & 0.0160 & 0.0045 & $\times$ 0.0114 \\
\textsf{b10}  & \textsf{point}   & 0.0280 & 0.0315 & --- \\
\textsf{b10}  & \textsf{spatial} & 0.0077 & 0.0131 & --- \\
\textsf{k5}   & \textsf{iid}     & 0.0164 & 0.0188 & --- \\
\textsf{k5}   & \textsf{point}   & 0.0261 & 0.0328 & --- \\
\textsf{k5}   & \textsf{spatial} & 0.0077 & 0.0236 & --- \\
\bottomrule
\end{tabular}
\end{table}

\subsection{Calibrating the slack}
\label{sec:slack}
We enumerate every certified comparison of Proposition~\ref{prop:ceilings} over
the seven numerically distinct configurations and three protocols: $87$ distinct
R1 tests (after collapsing the $18$ alias cells of the gate-closed runs), $42$ R2,
$63$ R3, $42$ R4 and $84$ R5, for $318$ orderings that the pipeline is obliged to
satisfy. Eighteen are inverted; Table~\ref{tab:slack} lists all of them. Hence
\begin{equation}
\hat s = 0.0220\ \text{macro-}\Fone,
\end{equation}
attained at \textsf{ctx}/\textsf{spatial}/$\MIBfm$ versus $\MONOsc$, i.e.\
\emph{in the headline configuration itself}. Restricting to ablation-dominance
pairs alone gives the narrower floor $\hat s_{\mathrm{abl}}=0.0154$, at
\textsf{b5}/\textsf{spatial}/$\MIBfm$ versus $\GEO$. No R2 inversion occurred:
the loosened majority-constant ceiling of Remark~\ref{rem:cons} was never
approached, with margins of $+0.0307$ to $+0.1150$.

For comparison, the fold standard deviation at the maximizing cell is $0.0516$
and the largest anywhere in the grid is $0.0889$
(\textsf{drv}/\textsf{spatial}/$\MONOnv$); the widest cluster bootstrap
half-width among the constrained-arm dominance intervals is $0.0793$
(\textsf{b5}/\textsf{spatial}/$\MONOsc$ versus $\GEO$). So
$\hat s$ is $0.43\times$ the local fold spread and $0.28\times$ the widest
interval --- but $0.87\times$ the half-width of the very quantity it invalidates,
the spatial \MONOsc{} price ($0.0254$). The point of $\hat s$ is not that it is
larger than a bootstrap interval; it is that it is derived from an ordering the
experiment must satisfy, so it survives in cells where the bootstrap is silent.
The most striking illustration is inversion \#17--\#18 in
Table~\ref{tab:slack}: at $10^\circ$ blocks the constrained \MONOnv{} beats its
own unconstrained superclass \MC{} by $0.0202$ with interval
$[-0.0128,+0.0529]$ --- a violation the bootstrap cannot see, in a cell where
$\hat\head=0.0050$ and nothing was measurable to begin with.

Consequences, stated as Corollary~\ref{cor:slack} prescribes:
\begin{enumerate}
\item \textbf{Reporting floor.} Four of the twenty-one price cells in
Table~\ref{tab:main} satisfy $|\hat\price|<\hat s$ and carry no interpretation:
\MONOnv{} ($0.0143$), \MONOsc{} ($0.0179$), \MIB{} ($-0.0096$) and \MIBfp{}
($-0.0130$), all under spatial blocking. All four would ordinarily be reported as
``the constraint is nearly free,'' and two of them as ``the constraint helps.''
Across the grid the same test disqualifies the entire spatial rung of \textsf{b5}
and \textsf{b10}.
\item \textbf{A priori futility.} Two of the twenty-one
(configuration $\times$ protocol) cells have $\hat\head<\hat s$:
\textsf{b5}/\textsf{spatial} ($\hat\head=0.0184$) and
\textsf{b10}/\textsf{spatial} ($\hat\head=0.0050$). Their ten shielded-arm price
cells were uninformative \emph{before any constrained model was trained}, so
Algorithm~\ref{alg:screen} would have terminated them after two unconstrained
fits. A third, \textsf{ctx}/\textsf{spatial} with $\hat\head=0.0427$, clears the
floor by only a factor of $1.9$.
\end{enumerate}

Two clarifications on interpretation. First, a negative measured price is not by
itself an error: hard constraints act as regularizers and can improve
generalization, and the fitted arm is not the minimizer of the out-of-fold
statistic over its class, so the population ordering $\price\ge0$ does not
transfer to estimates. The largest such effect in our grid,
\textsf{b10}/\textsf{spatial}/$\MIB$ beating \MC{} by $0.0270$ with interval
$[+0.0071,+0.0496]$ and $p=0.003$, is excluded from \eqref{eq:slack} anyway
because those two arms are not feature-matched (Remark~\ref{rem:notmatched}).
Second, and symmetrically, every inversion we \emph{do} count admits the same
regularization reading --- the ablated class is smaller still, so it too can win
by variance reduction. That is precisely why $\hat s$ should be read as a
resolution, i.e.\ the scale below which the pipeline cannot separate a genuine
excess-risk gap from an estimation artefact, rather than as a bug count.

\begin{table*}[!t]
\renewcommand{\arraystretch}{1.15}
\caption{Complete slack ledger: all $18$ inverted orderings among the $318$
certified comparisons of the grid. ``relation'' refers to
Proposition~\ref{prop:ceilings}. $\hat s=0.0220$ is the maximum of the last
column. Intervals are the paired $95\%$ cluster bootstrap difference of the
subclass against the superclass where the harness instruments that reference;
R3 and R4 are not instrumented against \MIB{} in the emitted artefacts and are
marked ---. Configurations \textsf{b2} and \textsf{glob} reproduce \textsf{ctx}
bytewise and would triple rows \#1, \#5, \#9, \#10; they are not listed.}
\label{tab:slack}
\centering
\footnotesize
\begin{tabular}{@{}rlllcccl@{}}
\toprule
\# & Config & Protocol & Comparison (subclass $\preceq$ superclass) & relation &
$\Fone$ sub & $\Fone$ super & inversion \\
\midrule
1  & \textsf{ctx}  & \textsf{iid}     & $\MIBfm \preceq \MONOsc$ & R3 & 0.7809 & 0.7843 & 0.0034 \\
2  & \textsf{b1}   & \textsf{iid}     & $\MIBfm \preceq \MONOsc$ & R3 & 0.7809 & 0.7843 & 0.0034 \\
3  & \textsf{b5}   & \textsf{iid}     & $\MIBfm \preceq \MONOsc$ & R3 & 0.7809 & 0.7843 & 0.0034 \\
4  & \textsf{b10}  & \textsf{iid}     & $\MIBfm \preceq \MONOsc$ & R3 & 0.7809 & 0.7843 & 0.0034 \\
5  & \textsf{ctx}  & \textsf{iid}     & $\MIBfp \preceq \MIB$    & R4 & 0.7968 & 0.7888 & 0.0080 \\
6  & \textsf{b1}   & \textsf{iid}     & $\MIBfp \preceq \MIB$    & R4 & 0.7968 & 0.7888 & 0.0080 \\
7  & \textsf{b5}   & \textsf{iid}     & $\MIBfp \preceq \MIB$    & R4 & 0.7968 & 0.7888 & 0.0080 \\
8  & \textsf{b10}  & \textsf{iid}     & $\MIBfp \preceq \MIB$    & R4 & 0.7968 & 0.7888 & 0.0080 \\
9  & \textsf{ctx}  & \textsf{spatial} & $\MIBfm \preceq \MONOsc$ & R3 & 0.5630 & 0.5850 & \textbf{0.0220} \\
10 & \textsf{ctx}  & \textsf{spatial} & $\MIBfp \preceq \MIB$    & R4 & 0.6159 & 0.6125 & 0.0034 \\
11 & \textsf{b1}   & \textsf{spatial} & $\MIBfm \preceq \MONOsc$ & R3 & 0.5982 & 0.6078 & 0.0096 \\
12 & \textsf{b1}   & \textsf{spatial} & $\MIBfp \preceq \MIB$    & R4 & 0.6381 & 0.6182 & 0.0199 \\
13 & \textsf{b5}   & \textsf{spatial} & $\MONOsc \preceq \GEO$   & R1 & 0.5625 & 0.5670 & 0.0045 \\
14 & \textsf{b5}   & \textsf{spatial} & $\MIBfm \preceq \GEO$    & R1 & 0.5516 & 0.5670 & 0.0154 \\
15 & \textsf{b5}   & \textsf{spatial} & $\MIBfm \preceq \MONOsc$ & R3 & 0.5516 & 0.5625 & 0.0109 \\
16 & \textsf{b5}   & \textsf{spatial} & $\MIBfp \preceq \MIB$    & R4 & 0.5847 & 0.5696 & 0.0151 \\
17 & \textsf{b10}  & \textsf{spatial} & $\MONOnv \preceq \MC$    & R5 & 0.5419 & 0.5218 & 0.0202 \\
18 & \textsf{b10}  & \textsf{spatial} & $\MONOsc \preceq \MC$    & R5 & 0.5357 & 0.5218 & 0.0139 \\
\bottomrule
\end{tabular}
\end{table*}

\subsection{The screening-leak probe returns a null}
\label{sec:leak}
The \textsf{glob} configuration differs from the headline run only in deciding the
exogeneity gate on the full dataset rather than on each training fold. The
resulting inflation is
\[
\Delta\Fone(\MONOsc)=+0.0000,\qquad \Delta\Fone(\MIB)=+0.0000,
\]
exactly, in all three protocols; \textsf{glob} reproduces \textsf{ctx} bytewise on
every arm. Two reasons combine. First, the gate is unanimous: insolation was
admitted in $15/15$ training folds, so pooling the gate decision changes nothing.
Second, in this implementation the pooled-screening flag governs \emph{only} the
gate; the sign pattern $\sigma^{\mathrm{sc}}$ is screened per training fold in
both arms by construction. The probe therefore covers leakage in channel
\emph{admission} only, and not the more consequential error of screening the sign
pattern itself on pooled data, which this design forecloses rather than measures.

The \textsf{b2} configuration, identical to \textsf{ctx} by construction,
likewise reproduces it bytewise on all $33$ model-protocol cells, confirming that
the harness is deterministic under seed $0$ and that no result in this paper
depends on run-to-run variation in fitting.

\section{Discussion}
\subsection{What should be reported}
Any claim about the accuracy cost of a domain constraint should be accompanied by
the triple $(\hat\price,\hat\head,\hat s)$ and an explicit statement of the free
feature set. The first number alone is uninterpretable: our headline price of
$0.0473$ becomes $0.3470$ when the shield is removed and $-0.0139$ when the
blocks are coarsened, with the physics held fixed. The second bounds it and is
nearly free to compute, being the standard ablation. The third says whether either
is above the noise; in our grid it disqualifies four of twenty-one headline price
cells and ten cells before training. For hybrid architectures with a fixed mixing
weight between learned and prescribed components~\cite{kriuk2026permafrost},
$\head$ should be computed with respect to the covariates entering the prescribed
component. For models with learnable physical
parameters~\cite{kriuk2026poseidon}, the relevant $S$ is the set of covariates
whose functional form the physical law pins down. In both cases the ablation is
already implemented; only the comparison is missing.

Two secondary recommendations follow from findings we did not anticipate. First,
report the granularity of the bootstrap clustering alongside any claim that a
driver's effect sign is ``identified,'' since refining our blocks from $10^\circ$
to $1^\circ$ moves temperature from $6/15$ to $9/15$ identified folds without
touching the data. Second, report the label resolution, since moving from $K=3$ to
$K=4$ raised every driver's mean partial effect and every normalized price while
leaving the ceiling essentially fixed --- the prior became measurably more costly
purely because the target became finer.

\subsection{On reporting our own negative cells}
Four of our twenty-one headline price cells fall below $\hat s$ and are reported as
uninterpretable rather than as evidence that the prior is cheap. This is not a
weakness of the instantiation but the intended use of Corollary~\ref{cor:slack}: a
screen that never fires on its author's own grid provides no evidence that it fires
at all. The same applies to the seven gate-open cells in which the spread exceeds
its bound (Table~\ref{tab:ctx}): each excess decomposes exactly into two
elementary inversions already counted in $\hat s$, which is what a
resolution-limited measurement of a small quantity is expected to look like.

\subsection{Why constraints can still be worth imposing}
Nothing above argues against shape constraints. Theorem~\ref{thm:main} caps the
in-distribution price and says nothing about extrapolation --- where the
$S$-blind surrogate that made the constraint look cheap is precisely the component
that fails, since a coordinate-based shield cannot transfer to unseen
territory~\cite{meyer2021}. Our own numbers make this concrete: the arm that
recovers $92.9\%$ of full-model accuracy under spatial blocking does so using
latitude and longitude, which carry no mechanism whatsoever and cannot be
evaluated outside the sampled domain. That is exactly the motivation for
prescribing a physical sensitivity in the permafrost
setting~\cite{kriuk2026permafrost}: the prescribed term is not there to improve
interpolation. Monotonicity additionally buys auditability and satisfies
requirements no accuracy metric encodes~\cite{wang2020deontological}; our
constrained arms carry a certificate that survives the deployed residualizer
exactly ($\VR_{\text{own}}=0$ end-to-end), which is a property of the artefact,
not of its macro-$\Fone$. Our claim is narrower: \emph{in-distribution accuracy
cost is the wrong instrument for evaluating a physical prior}, because its dynamic
range is set by the ablation gap rather than by the prior.

\section{Conclusion}
A shape constraint on a feature can never cost more than the feature is worth,
because ignoring the feature is always an admissible way to satisfy the
constraint. Operationally: \emph{a constrained model must never be beaten by its
own ablation}. This elementary bound turns the widely reported ``cost of physics''
into a quantity whose scale is fixed by the ablation gap, and hence by the free
feature set and the validation protocol rather than by the physics. On an ordinal
wildfire-severity task the same monotone prior costs $0.0473$ macro-$\Fone$ with
coordinates left free and $0.3470$ without them; the certified ceiling that bounds
both falls from $0.1288$ to $0.0050$ as the spatial blocks coarsen, taking every
price with it; the unconstrained model violates the prior at rate $0.48$ while
enforcing it costs $0.02$; and the pipeline's own inversions of orderings it is
obliged to satisfy calibrate a resolution of $\hat s=0.0220$ that renders four of
twenty-one headline price cells uninterpretable and ten cells of the grid
unidentifiable before a single constrained model is trained. We recommend
replacing the single reported cost with the triple (price, ceiling, slack) plus
the free feature set, and screening constraint-cost experiments with two
unconstrained fits before training any constrained model.


\end{document}